\documentclass{article}
\usepackage{ijcai23}

\usepackage{times}
\usepackage{soul}
\usepackage{url}
\usepackage[utf8]{inputenc}
\usepackage[small]{caption}
\usepackage{graphicx}
\usepackage{amsthm}
\usepackage{booktabs}
\usepackage{algorithm}
\usepackage{algorithmic}
\usepackage[switch]{lineno}

\usepackage[nokeyprefix]{refstyle}
\usepackage[hidelinks]{hyperref}
\newcommand{\citet}[1]{\citeauthor{#1} \shortcite{#1}}
\usepackage{textcomp}
\usepackage{amsfonts}
\usepackage{amsmath}
\usepackage[capitalise]{cleveref}
\usepackage{booktabs}
\usepackage{multirow}
\usepackage{subcaption}
\usepackage{adjustbox}
\usepackage{titlesec}
\titleformat{\subsubsection}[runin]{\normalfont\normalsize\bfseries}{\thesubsubsection}{1em}{}[.\;\;]
\usepackage{tikz}
\usepackage{pgfplots}
\pgfplotsset{compat=1.17}
\usepgfplotslibrary{fillbetween}
\usetikzlibrary{fit}

\newtheorem{theorem}{Theorem}
\newtheorem{definition}{Definition}

\newtheorem{lemma}[theorem]{Lemma}
\title{Recognizable Information Bottleneck}

\author{
Yilin Lyu$^1$\and
Xin Liu$^1$\and
Mingyang Song$^1$\and
Xinyue Wang$^1$\and\\
Yaxin Peng$^2$\and
Tieyong Zeng$^{3}$\And
Liping Jing$^1$\thanks{Corresponding author.}
\affiliations
$^1$Beijing Key Lab of Traffic Data Analysis and Mining, Beijing Jiaotong University\\
$^2$Department of Mathematics, School of Science, Shanghai University\\
$^3$The Chinese University of Hong Kong
\emails
\{yilinlyu, xin.liu, mingyang.song, xinyuewang, lpjing\}@bjtu.edu.cn,\\
yaxin.peng@shu.edu.cn,
zeng@math.cuhk.edu.hk
}
\begin{document}

\maketitle
\let\thefootnote\relax\footnote{Preprint.}

\begin{abstract}
    Information Bottlenecks (IBs) learn representations that generalize to unseen data by information compression. However, existing IBs are practically unable to guarantee generalization in real-world scenarios due to the vacuous generalization bound. The recent PAC-Bayes IB uses information complexity instead of information compression to establish a connection with the mutual information generalization bound. However, it requires the computation of expensive second-order curvature, which hinders its practical application. In this paper, we establish the connection between the recognizability of representations and the recent functional conditional mutual information (\textit{f}-CMI) generalization bound, which is significantly easier to estimate. On this basis we propose a Recognizable Information Bottleneck (RIB) which regularizes the recognizability of representations through a recognizability critic optimized by density ratio matching under the Bregman divergence. Extensive experiments on several commonly used datasets demonstrate the effectiveness of the proposed method in regularizing the model and estimating the generalization gap.
\end{abstract}
\section{Introduction}
\label{sec:intro}
Learning representations that are generalized to unseen data is one of the most fundamental goals for machine learning. The IB method \cite{DBLP:journals/corr/physics-0004057} is one of the prevailing paradigms. It centers on learning representations that not only sufficiently retain information about the label, but also compress information about the input.
The rationale behind the IB is that the information compression of representations penalizes complex hypotheses and guarantees generalization by a probability generalization bound \cite{DBLP:journals/tcs/ShamirST10,DBLP:conf/isit/VeraPV18}. It is expected that one can obtain a lower generalization gap through optimizing the IB training objective.
However, in finite and continuous data scenarios, such as an image classification task with deep neural networks (DNNs), this upper bound is dominated by a prohibitively large constant with any reasonable probability and becomes vacuous, thus making the information compression term negligible \cite{Rodriguez_Galvez_2019}.
This deprives IB method of the theoretical guarantee in real-world scenarios.

Recently, motivated by the mutual information (MI)-based generalization bound in the information-theoretic framework \cite{DBLP:conf/nips/XuR17}, \citet{wang2022pacbayes} proposed to use the \textit{information stored in weights}, i.e., the MI between the input dataset and the output weights of the algorithm, to replace the compression term in the IB objective.
It is worth noting, however, that the MI-based generalization bound may still be vacuous due to the fact that the MI can easily go to infinity, e.g., in cases where the learning algorithm is a deterministic function of the input data \cite{DBLP:journals/jsait/BuZV20,DBLP:conf/colt/SteinkeZ20}.
More importantly, it is notoriously difficult to estimate the MI between two high-dimensional random variables, e.g., weights of model and a training dataset.
Although one can recover non-vacuous bound by restricting the prior and posterior of the weights to Gaussian distributions \cite{DBLP:conf/uai/DziugaiteR17,wang2022pacbayes}, the computational burdens (from e.g., Monte-Carlo sampling or calculating the second-order curvature information) are still cumbersome.

To address the unbounded nature of MI-based bound, \citet{DBLP:conf/colt/SteinkeZ20} proposed the celebrated conditional mutual information (CMI) generalization bound, which normalizes each datum to one bit, thus is always bounded. The intuition is that measuring information that ``recognizes'' an input sample is much simpler than measuring information that ``reconstructs'' an input sample.
\citet{DBLP:conf/nips/HarutyunyanRSG21} further extended this idea to a black-box algorithm setting and proposed a functional CMI (\textit{f}-CMI) generalization bound, which significantly reduces the computational cost by computing the CMI with respect to the low-dimensional predictions rather than the high-dimensional weights.
This inspires us to consider the DNN encoder as a black box and the representation as the output predictions, which coincides with the original IB setting.
Unfortunately, the computation of the \textit{f}-CMI bound requires multiple runs on multiple draws of data, there is not yet a practical optimization algorithm to incorporate it into the training of the models.

In this paper, we first formulate the \textit{recognizability} from the perspective of \textit{binary hypothesis testing} (BHT) and show the connection between the \textit{recognizability of representations} and the \textit{f}-CMI generalization bound. Based on the reasoning that lower recognizability leads to better generalization, we propose a new IB objective, called the Recognizable Information Bottleneck (RIB).
A recognizability critic is exploited to optimize the RIB objective using \textit{density ratio matching} under the Bregman divergence.
Extensive experiments on several commonly used datasets demonstrate the effectiveness of the RIB in controlling generalization and the potential of estimating the generalization gap for the learned models using the recognizability of representations.

Our main contributions can be concluded as follows:
\begin{itemize}
    \item We define a formal characterization of recognizability and then show the connection between the recognizability of representations and the \textit{f}-CMI bound.
    \item We propose the RIB objective to regularize the recognizability of representations.
    \item We propose an efficient optimization algorithm of RIB which exploits a recognizability critic optimized by density ratio matching.
    \item We empirically demonstrate the regularization effects of RIB and the ability of estimating the generalization gap by the recognizability of representations.
\end{itemize}

\subsection{Other Related Work}
\phantomsection
\subsubsection{Information bottlenecks}
There have been numerous variants of IB \cite{DBLP:conf/iclr/AlemiFD017,DBLP:journals/entropy/KolchinskyTW19,DBLP:journals/entropy/GalvezTS20,DBLP:conf/aaai/Pan00021,DBLP:conf/icassp/YuYP21} and have been utilized on a wide range of machine learning tasks \cite{DBLP:conf/eccv/DuXXQZS020,DBLP:conf/ijcai/Lai0ZLS021,DBLP:conf/aaai/LiSWZRLK022}. However, one pitfall of IBs is that there may be no causal connection between information compression and generalization in real-world scenarios \cite{DBLP:conf/iclr/SaxeBDAKTC18,Rodriguez_Galvez_2019,DBLP:conf/nips/DuboisKSV20}. \citet{wang2022pacbayes} addressed this issue by introducing an information complexity term to replace the compression term, which turns out to be the Gibbs algorithm stated in the PAC-Bayesian theory \cite{Catoni_2007,DBLP:conf/nips/XuR17}.
\phantomsection
\subsubsection{Information-theoretical generalization bounds} The MI between the input and output of a learning algorithm has been widely used to bound the generalization error \cite{DBLP:conf/nips/XuR17,DBLP:conf/isit/PensiaJL18,DBLP:conf/nips/NegreaHDK019,DBLP:journals/jsait/BuZV20}. To tackle the intrinsic defects of MI, \citet{DBLP:conf/colt/SteinkeZ20} proposed the CMI generalization bound, then \citet{DBLP:conf/nips/HaghifamNK0D20} proved that the CMI-based bounds are tighter than MI-based bounds. \citet{DBLP:journals/tit/ZhouTL22} made the CMI bound conditioned on an individual sample and obtained a tighter bound. The above bounds are based on  the output of the algorithm (e.g. weights of the model) rather than information of predictions which is much easier to estimate in practice \cite{DBLP:conf/nips/HarutyunyanRSG21}.

\section{Preliminaries}
\subsection{Notation and Definitions}
We use capital letters for random variables, lower-case letters for realizations and calligraphic letters for domain.
Let $X$ and $Y$ be random variables defined on a common probability space, the mutual information between $X$ and $Y$ is defined as $I(X;Y)=D_\mathrm{KL}(P_{X,Y} \Vert P_{X} \otimes P_{Y})$, where $D_\mathrm{KL}$ is the Kullback-Leibler (KL) divergence. The conditional mutual information is defined as $I(X;Y|Z)=D_\mathrm{KL}(P_{X,Y|Z} \Vert P_{X|Z} P_{Y|Z}|P_Z)=\mathbb{E}_{z\sim P_Z}\left[ I(X;Y|Z=z) \right]$.
A random variable $X$ is $\sigma$-subgaussian if 
$\mathbb{E}\left[ \exp (\lambda(X-\mathbb{E}(X))) \right] \le \exp \left({\lambda^2 \sigma^2}/{2}\right)$ for all $\lambda \in \mathbb{R}$.

Throughout this paper, we consider the supervised learning setting.
There is an instance random variable $X$, a label random variable $Y$, an unknown data distribution $\mathcal{D}$ over the $\mathcal{Z} = \mathcal{X} \times \mathcal{Y}$ and a training dataset $S=(Z_1,Z_2,\dots,Z_n)\sim \mathcal{D}^n$ consisting of $n$ i.i.d. samples drawn from the data distribution.
There is a learning algorithm $A: \mathcal{Z}^n \rightarrow \mathcal{W}$ and an encoder function $f_{\theta}: \mathcal{X} \rightarrow \mathcal{T}$ from its class $\mathcal{F}:=\{T=\mathbb{E}_{X}[\mathbf{1}[t=f_{\theta}(X)]]:\theta \in \mathcal{W}\}$. We assume that $\mathcal{T}$ is of finite size and consider $T$ as the representation.
We define two expressions of generalization, one from the ``weight'' perspective and one from the ``representation'' perspective:
\begin{enumerate}
    \item Given a loss function $\ell: \mathcal{W} \times \mathcal{Z} \rightarrow \mathbb{R}$, the true risk of the algorithm $A$ w.r.t. $\ell$ is $\mathcal{L}(w,S)=\mathbb{E}_{z^\prime \sim \mathcal{D}} \ell(w,z^\prime)$ and the empirical risk of the algorithm $A$  w.r.t. $\ell$ is $\mathcal{L}_{\mathrm{emp}}(w,S)=\frac{1}{n}\sum_{i=1}^{n} \ell(w,z_i)$. The generalization error is defined as $\mathrm{gen}(w,S)=\mathcal{L}(w,S) - \mathcal{L}_{\mathrm{emp}}(w,S)$;
    \item Given a loss function $\ell: \mathcal{T} \times \mathcal{Y} \rightarrow \mathbb{R}$, the true risk of the algorithm $A$ w.r.t. $\ell$ is $\mathcal{L}(T,Y)=\mathbb{E}_{z^\prime \sim \mathcal{D}} \ell(t^\prime,y^\prime)$ and the empirical risk of the algorithm $A$  w.r.t. $\ell$ is $\mathcal{L}_{\mathrm{emp}}(T,Y)=\frac{1}{n}\sum_{i=1}^{n} \ell(t_i,y_i)$. The generalization error is defined as $\mathrm{gen}(T,Y)=\mathcal{L}(T,Y) - \mathcal{L}_{\mathrm{emp}}(T,Y)$.
\end{enumerate}

\subsection{Generalization Guarantees for Information Bottlenecks}

The IB method \cite{DBLP:journals/corr/physics-0004057} wants to find the optimal representation by minimizing the following Lagrangian:
\begin{equation}
    \label{eq:ib}
    \mathcal{L}_{\mathrm{IB}} = \mathcal{L}_{\mathrm{emp}}(T,Y) + \beta I(T;X),
\end{equation}
where $\beta$ is the Lagrange multiplier.
The $I(T;X)$ in \cref{eq:ib} can be seen as a compression term that regularizes the complexity of $T$. It has been demonstrated to upper-bound the generalization error \cite{DBLP:journals/tcs/ShamirST10,DBLP:conf/isit/VeraPV18}. Concretely, for any given $\delta > 0$, with probability at least $1-\delta$ it holds that
\begin{equation*}
    \label{eq:ib_bound}
    \left|\mathrm{gen}(T, Y)\right| \le \mathcal{O}(\frac{\log n}{\sqrt{n}}) \sqrt{I(T;X)} + C_{\delta},
\end{equation*}
where $C_{\delta}=\mathcal{O}(|\mathcal{T}|/\sqrt{n})$ is a constant that depends on the cardinality of the space of $T$. The bound implies that \textbf{the less information the representations can provide about the instances, the better the generalization of the algorithm}.
However, since $|\mathcal{T}|$ will become prohibitively large with any reasonable choice of $\delta$, this bound is actually vacuous in practical scenarios \cite{Rodriguez_Galvez_2019}.

A recent approach PAC-Bayes IB (PIB) \cite{wang2022pacbayes} tried to amend the generalization guarantee of IB by replacing the compression term with the information stored in weights, which leads to the following objective:
\begin{equation*}
    \label{eq:pib}
    \mathcal{L}_{\mathrm{PIB}} = \mathcal{L}_{\mathrm{emp}}(w,S) + \beta I(w;S).
\end{equation*}
The rationale behind the PIB is stem from the input-output mutual information generalization bound which can also be derived from the PAC-Bayesian perspective \cite{DBLP:conf/nips/XuR17,DBLP:conf/isit/BanerjeeM21}. Concretely, assume that the loss function $\ell(w,Z)$ is $\sigma$-subgaussian for all $w \in \mathcal{W}$, the expected generalization error can be upper-bounded by
\begin{equation*}
    \label{eq:pib_bound}
    \left| \mathbb{E}\left[\mathrm{gen}(w,S)\right] \right| \le \sqrt{\frac{2\sigma^2}{n} I(W;S)}.
\end{equation*}
Intuitively, the bound implies that \textbf{the less information the weight can provide about the input dataset, the better the generalization of the algorithm}.
However, as we mentioned above, the MI-based bound could be vacuous without proper assumptions and is hard to estimate for two high-dimensional variables such as the weight $W$ and the training dataset $S$.
\subsection{The CMI-based Generalization Bounds}
The CMI-based bounds \cite{DBLP:conf/colt/SteinkeZ20,DBLP:conf/nips/HarutyunyanRSG21} characterize the generalization by measuring the ability of recognizing the input data given the algorithm output. This can be accomplished by recognizing the training samples from the ``ghost'' (possibly the test) samples.
Formally, consider a supersample $\tilde{Z}\in \mathcal{Z}^{n \times 2}$ constructed from $n$ input samples mixed with $n$ ghost samples. A selector variable $U \sim \mathrm{Unif}(\{0,1\}^n)$ selects the input samples from the supersample $\tilde{Z}$ to form the training set $S=\tilde{Z}_U=(\tilde{Z}_{i,U_i+1})_{i=1}^n$. Given a loss function $\ell(w,z) \in [0,1]$ for all $w \in \mathcal{W}$ and all $z \in \mathcal{Z}$,
\citet{DBLP:conf/colt/SteinkeZ20} bound the expected generalization error by the CMI between the weight $W$ and the selector variable $U$ given the supersample $\tilde{Z}$:
\begin{equation}
    \label{eq:cmi_bound}
    \left| \mathbb{E}\left[\mathrm{gen}(w,S)\right] \right| \le \sqrt{\frac{2}{n} I(W;U|\tilde{Z})}.
\end{equation}
The intuition is that \textbf{the less information the weight can provide to recognize input samples from their ghosts, the better the generalization of the algorithm}.
In this way, the algorithm that reveals a datum has only CMI of one bit, bounding the CMI by $n \log 2$, while the MI may be infinite.
\citet{DBLP:conf/nips/HarutyunyanRSG21} extended this idea to a black-box algorithm setting which measures the CMI with respect to the predictions of the model rather than the weight. It is able to further extend the prediction domain to the representation domain, which yields a setting similar to the original IB.
Given a loss function $\ell(t,y) \in [0,1]$ for all $t \in \mathcal{T}$ and all $z \in \mathcal{Z}$, the expected generalization error is upper-bounded by the following \textit{f}-CMI bound:
\begin{equation}
    \label{eq:rib_bound}
    \left| \mathbb{E}\left[\mathrm{gen}(T,Y)\right] \right| \le \sqrt{\frac{2}{n} I(T;U|\tilde{Z})}.
\end{equation}
From the Markov chain $U \rightarrow W \rightarrow T$ and the data processing inequality, it follows that $I(T;U|\tilde{Z}) \le I(W;U|\tilde{Z})$, which indicates the bound in \cref{eq:rib_bound} is tighter than the bound in \cref{eq:cmi_bound}. Similarly, we can learn from the bound that \textbf{the less information the representations can provide to recognize input samples from their ghosts, the better the generalization of the algorithm}. Since the representation $T$ and the selector variable $U$ are relatively lower dimensional than the weight $W$ and the dataset $S$, respectively, estimating the \textit{f}-CMI bound is far more efficient than the previous bounds.

\section{Methodology}
Although the \textit{f}-CMI bound is theoretically superior to the generalization bounds of existing IBs, it has so far only been used as a risk certificate for the learned model \cite{DBLP:conf/nips/HarutyunyanRSG21}. Since computing the \textit{f}-CMI bound requires multiple runs on multiple draws of data, there is no immediate way to incorporate it into the training of deep models.
In this section, we first give a formal characterization of the recognizability and describe how to use the recognizability of representations to approximate the \textit{f}-CMI bound. This motivates us to propose the RIB objective and an efficient optimization algorithm using a recognizability critic.

\subsection{Recognizability}
In order to characterize the recognizability, we first review some basic concepts of binary hypothesis testing (BHT) \cite{Levy_2008}.
Assume that there are two possible hypotheses, $H_0$ and $H_1$, corresponding to two possible distributions $P$ and $Q$ on a space $\mathcal{X}$, respectively,
$$ H_0: X \sim P \qquad H_1: X \sim Q. $$
A binary hypothesis test between two distributions is to choose either $H_0$ or $H_1$ based on the observation of $X$.
Given a decision rule $\Xi: \mathcal{X} \rightarrow \{0,1\}$ such that $\Xi=0$ denotes the test chooses $P$, and $\Xi=1$ denotes the test chooses $Q$.
Then the recognizability can be defined as follows:
\begin{definition}[Recognizability]
    Let $H_0,H_1,P,Q,X,\Xi$ be defined as above. The recognizability between $P$ and $Q$ w.r.t. $\Xi$ for all randomized tests $P_{\Xi | X}$ is defined as
    \begin{multline*}
        \Re_\Xi(P,Q) := \\
        \lvert \left\{(\mathrm{Pr}(\Xi=1 | H_0), \mathrm{Pr}(\Xi=1 | H_1)) \mid \forall \Xi \sim P_{\Xi | X} \right\} \rvert,
    \end{multline*}
    where $\mathrm{Pr}(\Xi=1 | H_0)$ is the probability of error given $H_1$ is true and $\mathrm{Pr}(\Xi=1 | H_1)$ is the probability of success given $H_0$ is true.
\end{definition}

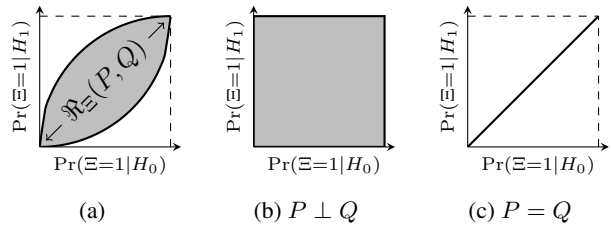
\begin{figure}
    \centering
    \begin{subfigure}[b]{0.16\textwidth}
        \centering
        \begin{tikzpicture}
            \begin{axis}[xmin=0, xmax=1.08, ymin=0, ymax=1.08,
                    ticks=none, axis lines=left, scale=0.33,
                    axis equal image,
                    xlabel = {$\scriptstyle \mathrm{Pr}(\Xi=1 | H_0)$},
                    ylabel = {$\scriptstyle \mathrm{Pr}(\Xi=1 | H_1)$},
                    ylabel near ticks,
                    xlabel near ticks,
                ]
                \addplot [name path=A,black,domain=0:1,smooth,thick] {(1-(x-1)^2)^0.5};
                \addplot [name path=B,black,domain=0:1,smooth,thick] {-(1-x^2)^0.5+1};
                \addplot [gray!50] fill between[of=A and B];
                \addplot [dashed] coordinates {(1,0)  (1,1)};
                \addplot [dashed] coordinates {(0,1)  (1,1)};
                \addplot [<->, black] coordinates {(0.05,0.05)  (0.95,0.95)};
                \node[rotate=45,fill=gray!50] at (axis cs:0.5,0.5) {$\Re_\Xi(P,Q)$};
            \end{axis}
        \end{tikzpicture}
        \caption{}
        \label{fig:recog1}
    \end{subfigure}%
    \begin{subfigure}[b]{0.16\textwidth}
        \centering
        \begin{tikzpicture}
            \begin{axis}[xmin=0, xmax=1.08, ymin=0, ymax=1.08,
                    ticks=none, axis lines=left, scale=0.33,
                    axis equal image,
                    xlabel = {$\scriptstyle \mathrm{Pr}(\Xi=1 | H_0)$},
                    ylabel = {$\scriptstyle \mathrm{Pr}(\Xi=1 | H_1)$},
                    ylabel near ticks,
                    xlabel near ticks,
                ]
                \draw[black,thick,fill=gray!50] (axis cs:0,0) rectangle (axis cs:1,1);
                % \addplot [dashed] coordinates {(1,0)  (1,1)};
                % \addplot [dashed] coordinates {(0,1)  (1,1)};
            \end{axis}
        \end{tikzpicture}
        \caption{$P \perp Q$}
        \label{fig:recog2}
    \end{subfigure}%
    \begin{subfigure}[b]{0.16\textwidth}
        \centering
        \begin{tikzpicture}
            \begin{axis}[xmin=0, xmax=1.08, ymin=0, ymax=1.08,
                    ticks=none, axis lines=left, scale=0.33,
                    axis equal image,
                    xlabel = {$\scriptstyle \mathrm{Pr}(\Xi=1 | H_0)$},
                    ylabel = {$\scriptstyle \mathrm{Pr}(\Xi=1 | H_1)$},
                    ylabel near ticks,
                    xlabel near ticks,
                ]
                \addplot [black,thick] coordinates {(0,0)  (1,1)};
                \addplot [dashed] coordinates {(1,0)  (1,1)};
                \addplot [dashed] coordinates {(0,1)  (1,1)};
            \end{axis}
        \end{tikzpicture}
        \caption{$P=Q$}
        \label{fig:recog3}
    \end{subfigure}
    \caption{Illustrations of recognizability.}
    \label{fig:recog}
\end{figure}

\phantomsection
\subsubsection{Remark}
In words, the recognizability is the area of the achievable region for all tests drawn from $P_{\Xi|X}$ (see \cref{fig:recog1}).
For BHT, one would like to select a test such that $\mathrm{Pr}(\Xi=1 | H_0)$ is as close to zero as possible while $\mathrm{Pr}(\Xi=1 | H_1)$ is as close to one as possible. However, the extreme case, i.e., $(\mathrm{Pr}(\Xi=1 | H_0)=0,\mathrm{Pr}(\Xi=1 | H_1)=1)$ or $(0,1)$, is achievable only if the hypotheses $H_0$ and $H_1$ are completely mutually exclusive, that is, $P$ is mutually singular w.r.t. $Q$, which indicates that the recognizability $\Re_\Xi(P,Q)=1$ (\cref{fig:recog2}). Similarly, when $H_0$ and $H_1$ are equivalent, that is, $P$ and $Q$ are identical, we have the recognizability $\Re_\Xi(P,Q)=0$ (\cref{fig:recog3}).

By the celebrated Neyman-Pearson lemma, we have that log-likelihood ratio testing (LRT) is the optimal decision rule, which means all the points in the achievable region are attained by LRT \cite{Neyman_Pearson_Pearson_1933}.
It is also well-known that the upper boundary between the achievable and unachievable region is the receiver operating characteristic (ROC) of LRT \cite{Levy_2008}. Therefore, given a LRT decision rule $\xi:\mathcal{X} \rightarrow [0,1]$, the recognizability w.r.t. $\xi$ can be computed by 
\begin{equation}
    \label{eq:recog_equation}
    \Re_\xi(P,Q) = 2 \times \mathrm{AUCROC}_\xi - 1,
\end{equation}
where $\mathrm{AUCROC}_\xi$ is the area under the ROC curve of $\xi$.

The formulation of recognizability offers us some fresh insights to consider generalization from the perspective of BHT, which leads to a new approach to regularization.
Concretely, we rewrite the \textit{f}-CMI term in \cref{eq:rib_bound} as
\begin{align}
    \label{eq:fcmi_2}
    I(T;U|\tilde{Z}) & =\mathbb{E}_{\tilde{z} \sim \tilde{Z}} D_\mathrm{KL}(P_{T,U|\tilde{Z}=\tilde{z}} \Vert P_{T|\tilde{Z}=\tilde{z}} \otimes P_{U})  \nonumber \\
                     & =\mathbb{E}_{\tilde{z} \sim \tilde{Z}} D_\mathrm{KL}(P_{T|U,\tilde{Z}=\tilde{z}} \Vert P_{T|\tilde{Z}=\tilde{z}}),
\end{align}
where \cref{eq:fcmi_2} follows by the uniformity of $U$. Consider a realization of supersample $\tilde{z}$. Let $t_{1},t_{2},\dots,t_{n}$ be a sequence drawn i.i.d. according to $P_{T|U,\tilde{Z}=\tilde{z}}$, then by the weak law of large numbers we have that
\begin{equation*}
    \label{eq:lrt}
    \frac{1}{n} \sum_{i=1}^{n} \log \frac{P_{T|U,\tilde{Z}=\tilde{z}}(t_{i})}{P_{T|\tilde{Z}=\tilde{z}}(t_i)} \rightarrow D_\mathrm{KL}(P_{T|U,\tilde{Z}=\tilde{z}} \Vert P_{T|\tilde{Z}=\tilde{z}}).
\end{equation*}
This suggests that the log-likelihood ratio between the distributions $P_{T|U,\tilde{Z}=\tilde{z}}$ and $P_{T|\tilde{Z}=\tilde{z}}$ amounts to the KL divergence, and eventually an estimate of $I(T;U|\tilde{z})$, which is the single-draw of the \textit{f}-CMI bound.
Thus, the more similar these two distributions are, the more difficult it is for us to determine which hypothesis $t_i$ originates from by LRT, namely, the lower the recognizability; we can also verify, in turn, that the KL divergence is minimal when we achieve the lowest recognizability (\cref{fig:recog3}).
To make this explicit, we theoretically establish the connection between recognizability of representations and the \textit{f}-CMI bound through the following theorem.
\begin{theorem}
    \label{prop1}
    Let $\tilde{z}$ be a realization of supersample $\tilde{Z}$, $T$ be the representation variable defined as above, $\xi:\mathcal{T} \rightarrow [0,1]$ be the LRT decision rule and $U \sim \mathrm{Unif}(\{0,1\}^n)$. The recognizability of representations $\Re_\xi(P_{T|\tilde{Z}=\tilde{z}},P_{T|U,\tilde{Z}=\tilde{z}})$ is upper bound by $I(T;U|\tilde{z}) + \log \frac{e}{2}$.
\end{theorem}
\begin{proof}
    See Appendix A.
\end{proof}

Inspired by the above insights, we propose the RIB to constrain the model to learn representations having low recognizability, which utilizes a log-likelihood ratio estimator that we name \textit{recognizability critic}. We will detail it in the next subsection.

\subsection{Recognizable Information Bottleneck}
We first define a new objective function that takes the recognizability of representations into consideration.
Let $\tilde{z}$ be a realization of $\tilde{Z}$ which is composed of a training set $S$ of size $n$ and a ghost set $\bar{S}$ of the same size, $U \sim \mathrm{Unif}(\{0,1\}^n)$ be a random variable indicating the set from which the instance of representation comes. Note that in general the ghost set is built from the validation or test (if available) set with dummy labels, and we will experiment later with ghost from different sources.
Now we can write a new IB objective called the Recognizable Information Bottleneck (RIB) as follows:
\begin{equation}
    \label{eq:rib}
    \mathcal{L}_{\mathrm{RIB}} = \mathcal{L}_{\mathrm{emp}}(T,Y) + \beta I(T;U|\tilde{z}),
\end{equation}
where the first term enforces empirical risk minimization, the second term regularizes the recognizability of representations and $\beta$ is the trade-off parameter.

To optimize the second term of \cref{eq:rib}, we first train a log-likelihood ratio estimator $V_{\phi}: \mathcal{T}^2 \rightarrow \mathbb{R}$ that we call the recognizability critic. It takes a pair of representations as input, one obtained from the training set $S$ samples and the other from the ghost set $\bar{S}$. The pair of representations is then randomly concatenated and treated as a single draw from the joint distribution $P_{T,U|\tilde{Z}=\tilde{z}}$, since the probability of occurrence of the possible outcomes is consistent.
Next we train the recognizability critic by maximizing the lower bound on the Jensen-Shannon divergence as suggested in \cite{DBLP:conf/iclr/HjelmFLGBTB19,DBLP:conf/icml/PooleOOAT19} by the following objective:
\begin{multline}
    \label{eq:recog_critic}
    \max_{\phi}
    \mathbb{E}_{\tilde{t} \sim P_{T|U,\tilde{Z}=\tilde{z}}} [-\mathrm{sp}(-V_{\phi}(\tilde{t}))] \\
    - \mathbb{E}_{\tilde{t} \sim P_{T|\tilde{Z}=\tilde{z}}} [\mathrm{sp}(V_{\phi}(\tilde{t}))],
\end{multline}
where $\mathrm{sp}(x)=\log (1+e^x)$ is the softplus function. Once the critic $V$ reaches the maximum, we can obtain the density ratio by
\begin{equation*}
    R(\tilde{t}) = \frac{P_{T|U,\tilde{Z}=\tilde{z}}(\tilde{t})}{P_{T|\tilde{Z}=\tilde{z}}(\tilde{t})} \approx \exp(V_{\phi}(\tilde{t})).
\end{equation*}
Since the lowest recognizability achieves when the two distributions are identical, that is, the density ratio $R^*\equiv 1$, to regularize the representations, we use a surrogate loss motivated by \textit{density-ratio matching} \cite{Sugiyama_Suzuki_Kanamori_2012}, which optimizes the Bregman divergence between the optimal density ratio $R^*$ and the estimated density ratio $R(\tilde{t})$ defined as
\begin{equation}
    \label{eq:bregman}
    D_{\mathrm{BR}_F}(R^{*} \Vert R(\tilde{t})) :=
    F(R^{*}) - F(R(\tilde{t})) - \partial F(R(\tilde{t}))(R^{*} - R(\tilde{t})).
\end{equation}
Here, $F$ is a differentiable and strictly convex function and $\partial F$ is its derivative. We can derive different measures by using different $F$. In the experiments we use $F(R)=R \log R - (1 + R) \log(1 + R)$, at which point the Bregman divergence is reduced to the binary KL divergence. More examples and comparison results are given in Appendix C.

Now the objective function for the encoder $f_\theta$ can be written as
\begin{equation}
    \label{eq:rib_opt}
    \min_{\theta} \frac{1}{n}\sum_{i=1}^{n} \ell(t_i,y_i) + \beta  D_{\mathrm{BR}}(R^{*} \Vert R(\tilde{t}_i)).
\end{equation}
The complete optimization algorithm is described in \cref{alg:algorithm}.
Note that although it is also feasible to directly minimize \cref{eq:recog_critic} with respect to the encoder parameter $\theta$ which leads to a minimax game similar to GAN \cite{DBLP:conf/nips/GoodfellowPMXWOCB14,DBLP:conf/nips/NowozinCT16}, we empirically find that this may cause training instability especially when data is scarce (as detailed in the next section).
We consider this is due to the randomness of $U$ which leads to large variance gradients and hinders the convergence of the model, while we utilize a fixed ratio for optimization thus it is not affected.

When the recognizability critic finishes training, we can obtain the class probability from the estimated density ratio via the sigmoid function, which gives us the following decision rules:
\begin{equation*}
    \hat{\xi}(\tilde{t}) = \frac{1}{1 + \exp(-R(\tilde{t}))}.
\end{equation*}
Then the recognizability of representations can be estimated based on the outputs of $\hat{\xi}$ via \cref{eq:recog_equation}.
We find that the recognizability of representations can be used as an indicator of the generalizability and achieve comparable performance to the \textit{f}-CMI bound, which will be discussed in the next section.

\begin{algorithm}[tb]
    \caption{Optimization of Recognizable Information Bottleneck (RIB)}
    \label{alg:algorithm}
    \textbf{Input}: training set $S=\{(x_i,y_i)\}_{i=1}^{N}$, ghost set $\bar{S}=\{\bar{x}_i\}_{i=1}^{N}$, encoder $f_{\theta}$, initial encoder parameter $\theta_0$, recognizability critic $V_{\phi}$, initial recognizability critic parameter $\phi_0$, batch size $B$, learning rate $\eta$ \\
    % \textbf{Parameter}: Optional list of parameters\\
    \textbf{Output}: trained encoder parameter $\theta$, trained recognizability critic parameter $\phi$
    \begin{algorithmic}[1] %[1] enables line numbers
        \STATE initialize: $\theta \leftarrow \theta_0, \phi \leftarrow \phi_0$
        \WHILE{not converged}
        \STATE sample training mini-batch $\{(x_i,y_i)\}_{i=1}^{B} \sim S$
        \STATE sample ghost mini-batch $\{\bar{x}_i\}_{i=1}^{B} \sim \bar{S}$
        \STATE sample $\{u_i\}_{i=1}^{B} \sim \mathrm{Unif}(\{0,1\})$
        \STATE compute $\{t_i=f_{\theta}(x_i)\}_{i=1}^{B}, \{\bar{t}_i=f_{\theta}(\bar{x}_i)\}_{i=1}^{B}$
        \STATE set $\tilde{t}=\{\tilde{t}_i\}_{i=1}^{B}$, where \\
        $
            \tilde{t}_i=
            \begin{cases}
                t_i \circ \bar{t}_i, & u_i=0 \\
                \bar{t}_i \circ t_i, & u_i=1
            \end{cases}
        $
        \STATE compute gradient $g_{\phi}$ by \cref{eq:recog_critic} and update $\phi$:\\
        $\phi \leftarrow \phi + \eta g_{\phi}$
        \STATE compute gradient $g_{\theta}$ by \cref{eq:rib_opt} and update $\theta$:\\
        $\theta \leftarrow \theta - \eta g_{\theta}$
        \ENDWHILE
        \STATE \textbf{return} ($\theta, \phi$)
    \end{algorithmic}
\end{algorithm}

\begin{table*}[t]
    \centering
    \begin{tabular}{lcccccccccccccccc}
        \toprule
        \multirow{2}{*}{Method} & \multicolumn{3}{c}{Fashion}              & \multicolumn{3}{c}{SVHN}                 & \multicolumn{3}{c}{CIFAR10}                                                                                                                                                                                                                                                                                 \\
        \cmidrule(lll){2-4} \cmidrule(lll){5-7} \cmidrule(lll){8-10}
        {}                      & 1250                                     & 5000                                     & 20000                                   & 1250                                     & 5000                                     & 20000                                    & 1250                                     & 5000                                     & 20000                                      \\
        \midrule
        CE                      & 18.61{\scriptsize \textpm 0.87}          & 14.57{\scriptsize \textpm 0.47}          & 10.18{\scriptsize \textpm 0.34}         & 30.03{\scriptsize \textpm 1.04}          & 19.51{\scriptsize \textpm 0.52}          & 11.77{\scriptsize \textpm 0.33}          & 61.51{\scriptsize \textpm 1.13}          & 50.22{\scriptsize \textpm 0.81}          & 32.81{\scriptsize \textpm 0.70}            \\
        L2                      & 18.54{\scriptsize \textpm 0.66}          & 14.81{\scriptsize \textpm 0.39}          & 10.57{\scriptsize \textpm 0.50}         & 29.60{\scriptsize \textpm 1.00}          & 19.30{\scriptsize \textpm 0.73}          & 11.19{\scriptsize \textpm 0.22}          & 61.62{\scriptsize \textpm 0.97}          & 49.89{\scriptsize \textpm 0.93}          & 32.41{\scriptsize \textpm 1.17}            \\
        Dropout                 & 18.58{\scriptsize \textpm 0.57}          & 14.58{\scriptsize \textpm 0.51}          & 10.20{\scriptsize \textpm 0.40}         & 29.68{\scriptsize \textpm 0.94}          & 19.40{\scriptsize \textpm 0.57}          & 11.82{\scriptsize \textpm 0.42}          & 61.53{\scriptsize \textpm 0.95}          & 50.02{\scriptsize \textpm 0.78}          & 32.69{\scriptsize \textpm 0.47}            \\
        VIB                     & 17.57{\scriptsize \textpm 0.44}          & 13.86{\scriptsize \textpm 0.35}          & 9.64{\scriptsize \textpm 0.24}          & 40.54{\scriptsize \textpm 5.75}          & 18.93{\scriptsize \textpm 0.61}          & 11.28{\scriptsize \textpm 0.23}          & 59.77{\scriptsize \textpm 0.78}          & 49.14{\scriptsize \textpm 0.60}          & 31.82{\scriptsize \textpm 0.39}            \\
        NIB                     & 18.30{\scriptsize \textpm 0.42}          & 14.34{\scriptsize \textpm 0.29}          & 9.94{\scriptsize \textpm 0.21}          & 29.34{\scriptsize \textpm 0.71}          & 19.44{\scriptsize \textpm 0.38}          & 11.59{\scriptsize \textpm 0.24}          & 61.00{\scriptsize \textpm 0.77}          & 50.01{\scriptsize \textpm 0.57}          & 32.32{\scriptsize \textpm 0.41}            \\
        DIB                     & 18.11{\scriptsize \textpm 0.37}          & 13.99{\scriptsize \textpm 0.34}          & 9.74{\scriptsize \textpm 0.27}          & 29.95{\scriptsize \textpm 0.93}          & 19.37{\scriptsize \textpm 0.62}          & 11.54{\scriptsize \textpm 0.23}          & 61.02{\scriptsize \textpm 0.70}          & 50.09{\scriptsize \textpm 0.66}          & 32.36{\scriptsize \textpm 0.49}            \\
        PIB                     & 18.44{\scriptsize \textpm 0.52}          & 14.33{\scriptsize \textpm 0.29}          & 10.00{\scriptsize \textpm 0.23}         & 29.40{\scriptsize \textpm 0.78}          & 19.40{\scriptsize \textpm 0.61}          & 11.60{\scriptsize \textpm 0.27}          & 61.19{\scriptsize \textpm 0.82}          & 50.09{\scriptsize \textpm 0.79}          & 32.48{\scriptsize \textpm 0.31}            \\
        RIB                     & \textbf{17.05{\scriptsize \textpm 0.41}} & \textbf{13.57{\scriptsize \textpm 0.28}} & \textbf{9.32{\scriptsize \textpm 0.25}} & \textbf{27.72{\scriptsize \textpm 1.17}} & \textbf{17.37{\scriptsize \textpm 0.52}} & \textbf{10.80{\scriptsize \textpm 0.27}} & \textbf{59.16{\scriptsize \textpm 1.17}} & \textbf{47.57{\scriptsize \textpm 0.60}} & \textbf{31.06{\scriptsize \textpm 0.38}}   \\
        \bottomrule
    \end{tabular}
    \caption{Comparison of the mean test error (\%) on three datasets with different training set sizes.}
    \label{table:perf}
\end{table*}

\section{Experiments}
In this section, we demonstrate the effectiveness of RIB as a training objective to regularize the DNNs and verify the capability of recognizability critic to estimate generalization gap. Code is available at \url{https://github.com/lvyilin/RecogIB}.
\subsection{Experimental Setup}
\phantomsection
\subsubsection{Datasets}
Our experiments mainly conduct on three widely-used datasets: Fashion-MNIST \cite{DBLP:journals/corr/abs-1708-07747}, SVHN \cite{netzer2011reading} and CIFAR10 \cite{krizhevsky2009learning}. We also give the results on MNIST and STL10 \cite{DBLP:journals/jmlr/CoatesNL11} in Appendix C.
We adopt the data setting of \cite{DBLP:conf/nips/HarutyunyanRSG21} to enable the calculation and comparison  of \textit{f}-CMI bound.
Concretely, each experiment is conducted on $k_1$ draws of the original training set (to form a supersample $\tilde{z}$) and $k_2$ draws of the train-val-split (i.e., $k_1 k_2$ runs in total), which correspond to the randomness of $\tilde{Z}$ and $T$, respectively. In all our experiments, $k_1=k_2=5$.
Unless otherwise stated, the validation set is used as the ghost set.
To demonstrate the effectiveness across different data scales, we perform sub-sampling on the training set with different sizes $n\in \{1250, 5000, 20000\}$.
\phantomsection
\subsubsection{Implementation details}
We use a DNN model composed of a 4-layer CNN (128-128-256-1024) and a 2-layer MLP (1024-512) as the encoder, and use a 4-layer MLP (1024-1024-1024-1) as the recognizability critic.
We train the learning model using Adam optimizer \cite{DBLP:journals/corr/KingmaB14} with betas of (0.9, 0.999) and train the recognizability critic using SGD with momentum of 0.9 as it is more stable in practice. All learning rates are set to 0.001, and the models are trained 100 epochs using the cosine annealing learning rate scheme with a batch size of 128. The trade-off parameter $\beta$ is selected from $\{10^{-1},10^0,10^1,10^2\}$ according to the desired regularization strength as discussed later. All the experiments are implemented with PyTorch and performed on eight NVIDIA RTX A4000 GPUs.
More experimental details are provided in Appendix B.

\subsection{Regularization Effects of RIB}
We report the average performance of RIB over $k_1 k_2$ draws to show its regularization effects on different data scales. We use cross entropy (CE) loss as the baseline. Furthermore, we compare with two common regularization methods: L2 normalization and dropout \cite{DBLP:journals/jmlr/SrivastavaHKSS14}. The strength of the L2 normalization is set to 1e-4 and the dropout rate is set to 0.1. We also compare with four popular IB counterparts: VIB \cite{DBLP:conf/iclr/AlemiFD017}, NIB \cite{DBLP:journals/entropy/KolchinskyTW19}, DIB \cite{DBLP:conf/icassp/YuYP21}, and PIB \cite{wang2022pacbayes}.
All methods impose regularization on the representations, except for PIB which is on the weights. Their regularization strength is determined by the best result of $\{10^{-5}, 10^{-4},\dots,10^{1} \}$.
\cref{table:perf} shows that RIB consistently outperforms the compared methods across all datasets and all training set sizes. This demonstrates that regularizing the recognizability of representations as well as reducing the recognizable information contributes significantly to improving the generalization of the model.
Another point worth mentioning is that the performance of PIB does not outperform VIB and its variants.
We consider that this is because the PIB computes the gradient for the overall data only once per epoch, in order to maintain an affordable computational complexity, while representation-based regularization methods compute the gradient for each iteration efficiently, thus provide more fine-grained information for model training.
Furthermore, we credit the improvement of RIB over VIB and its variants to the use of the ghost set, which enables the computation of recognizable information, and the reduction of recognizability of representations to facilitate generalization.

\begin{table}[ht]
    \centering
    \begin{tabular}{l|lll}
        \toprule
        Source   & 1250                                     & 5000                                     & 20000                                    \\
        \midrule
        SVHN     & 61.26{\scriptsize \textpm 1.20}          & 51.59{\scriptsize \textpm 0.97}          & 37.13{\scriptsize \textpm 0.94}          \\
        CIFAR100 & 60.31{\scriptsize \textpm 0.91}          & 49.20{\scriptsize \textpm 0.78}          & 33.27{\scriptsize \textpm 0.61}          \\
        CIFAR10  & \textbf{59.16{\scriptsize \textpm 1.17}} & \textbf{47.57{\scriptsize \textpm 0.60}} & \textbf{31.06{\scriptsize \textpm 0.38}} \\
        \midrule
        Baseline & 61.51{\scriptsize \textpm 1.13}          & 50.22{\scriptsize \textpm 0.81}          & 32.81{\scriptsize \textpm 0.70}          \\
        \bottomrule
    \end{tabular}
    \caption{Comparison of the mean test error (\%) using CIFAR10 as the training data and constructing ghosts with different sources.}
    \label{table:source}
\end{table}

\begin{figure}
    \centering
    \includegraphics[width=\columnwidth]{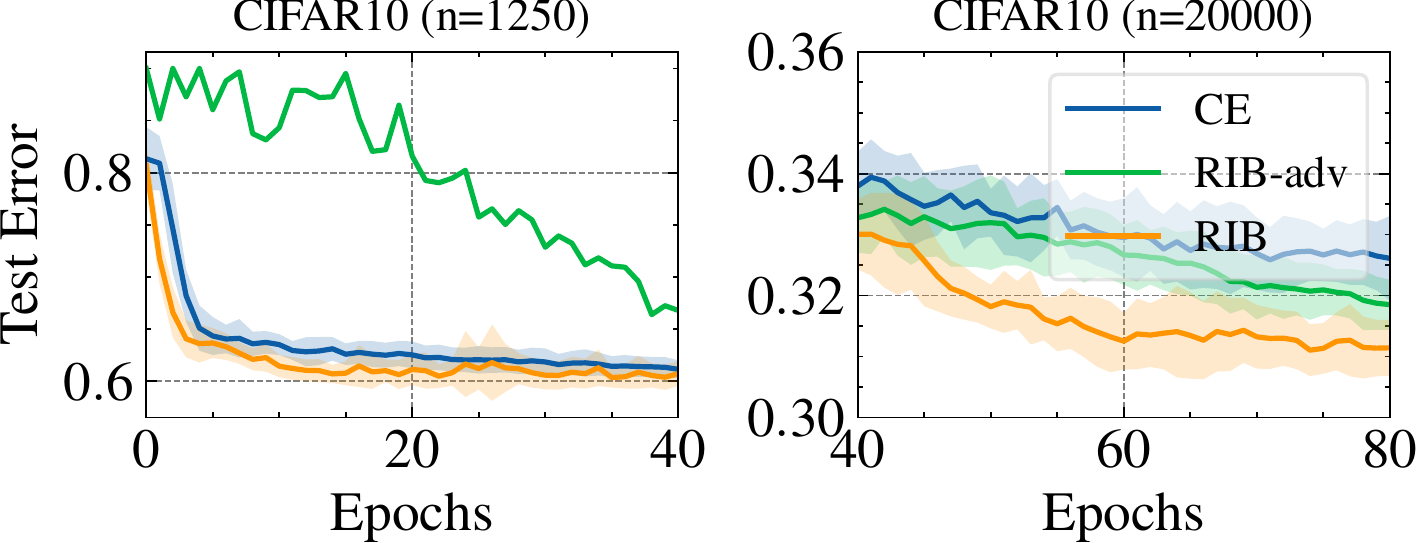}
    \caption{Mean test risk curves on CIFAR10 with training set sizes of 1250 and 20000. ``RIB-adv'' represents optimizing the RIB objective with adversarial training.}
    \label{fig:risk_curve}
\end{figure}
\phantomsection
\subsubsection{Effectiveness of the surrogate loss}
We first examine whether optimizing the Bregman divergence of \cref{eq:bregman} is better than optimizing \cref{eq:recog_critic} by adversarial training method using Jensen-Shannon divergence \cite{DBLP:conf/nips/NowozinCT16}. The risk curves are illustrated in \cref{fig:risk_curve}. We find that: \textit{i}) training RIB in an adversarial manner may produce unstable gradients when the data is scarce thus making the model difficult to converge; \textit{ii}) with more data, ``RIB-adv'' can also regularize the model and perform better than the CE baseline; and \textit{iii}) the RIB optimized by our surrogate loss consistently outperforms the baselines at different data scales, which demonstrates its effectiveness on regularizing the DNNs.

\phantomsection
\subsubsection{Results of ghosts with different sources} Although the construction of the ghost set is supposed to be performed on a supersample drawn from a single data domain, it is intriguing to investigate whether it can still achieve similar performance with different data sources. We use the CIFAR10 as the training set and construct the ghost from three sources: \textit{i}) SVHN, distinct from CIFAR10; \textit{ii}) CIFAR100 \cite{krizhevsky2009learning}, similar to CIFAR10 but with more categories and \textit{iii}) CIFAR10 for comparison. \cref{table:source} shows that \textit{i}) the more similar the sources are to the training set, the more effective the regularization will be. This is reasonable because recognizing samples from other domains may be too simple thus not provide useful information. More similar sources can provide more fine-grained recognizable information that is of more benefit to the model. \textit{ii}) Nevertheless, even data from different sources, such as SVHN, can still provide valuable information to improve generalization when training data is scarce (e.g., $n=1250$).

\begin{figure}
    \centering
    \includegraphics[width=\columnwidth]{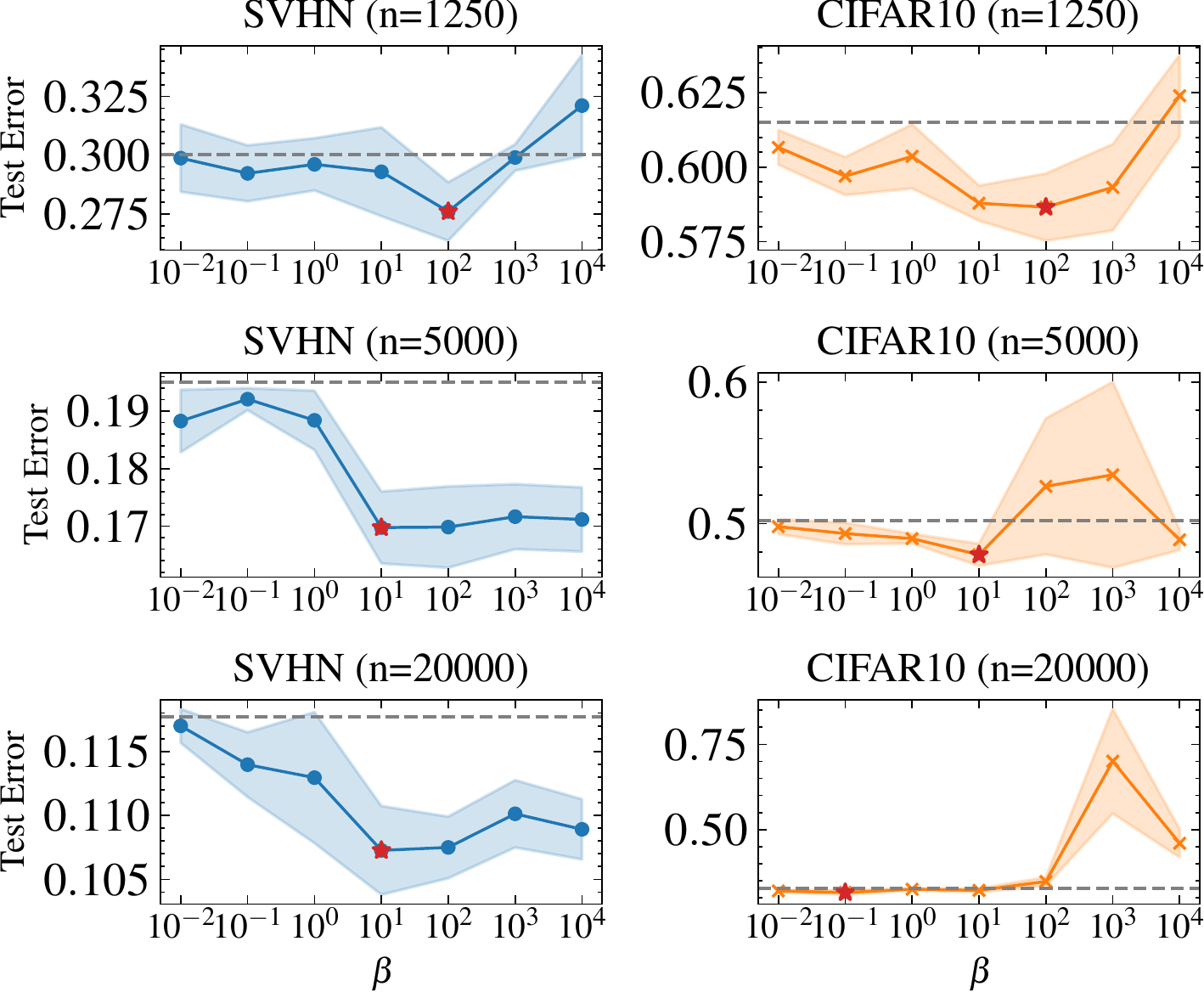}
    \caption{Impact of $\beta$ on mean test risk with different training set sizes. The horizontal dashed line indicates the performance on CE baseline. The star marker indicates the best result.}
    \label{fig:beta}
\end{figure}

\begin{figure*}
    \centering
    \begin{subfigure}[b]{0.3\textwidth}
        \centering
        \includegraphics[width=\textwidth]{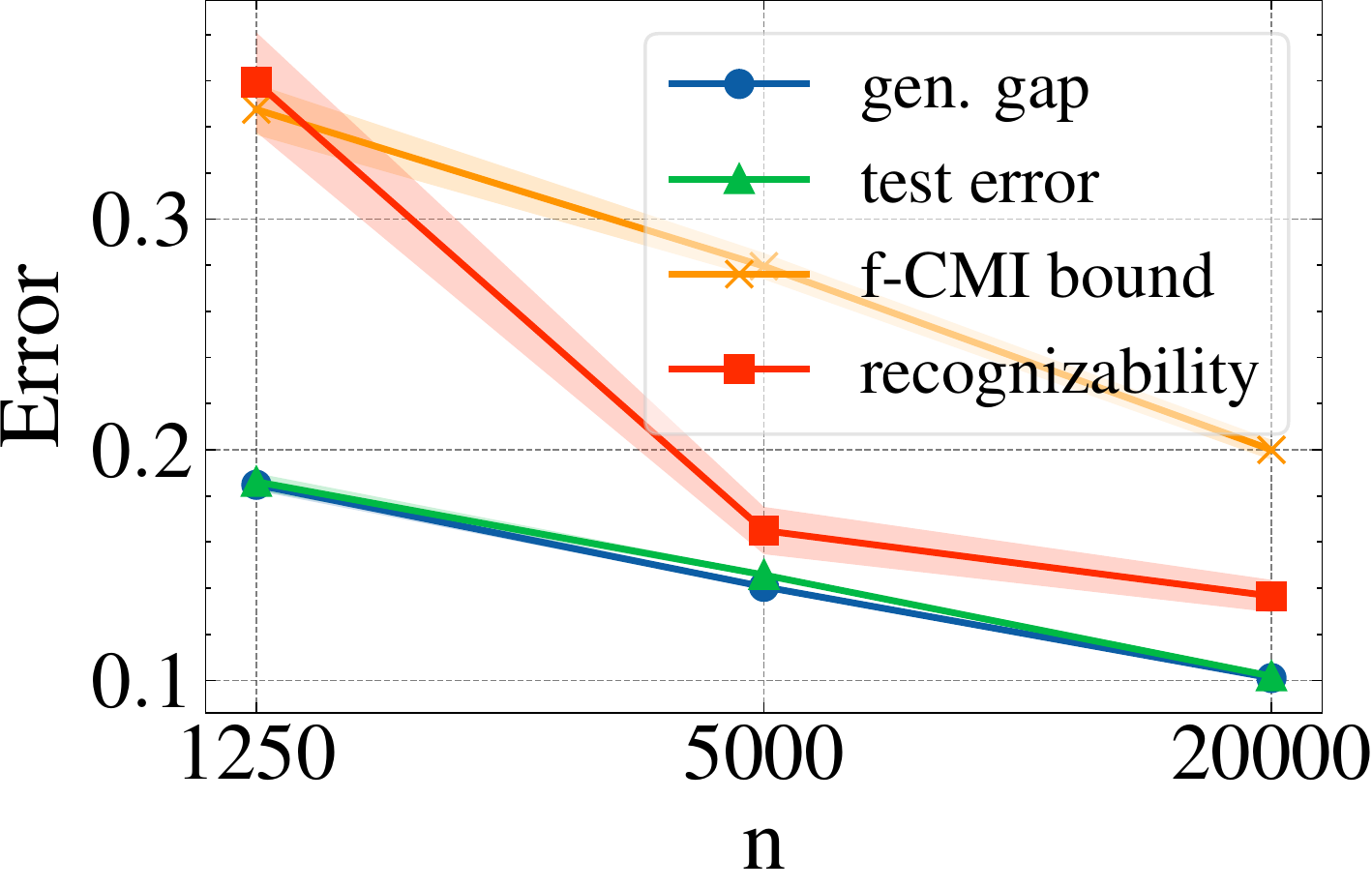}
        \caption{Fashion}
        \label{fig:gen_fashion}
    \end{subfigure}
    \hfill
    \begin{subfigure}[b]{0.3\textwidth}
        \centering
        \includegraphics[width=\textwidth]{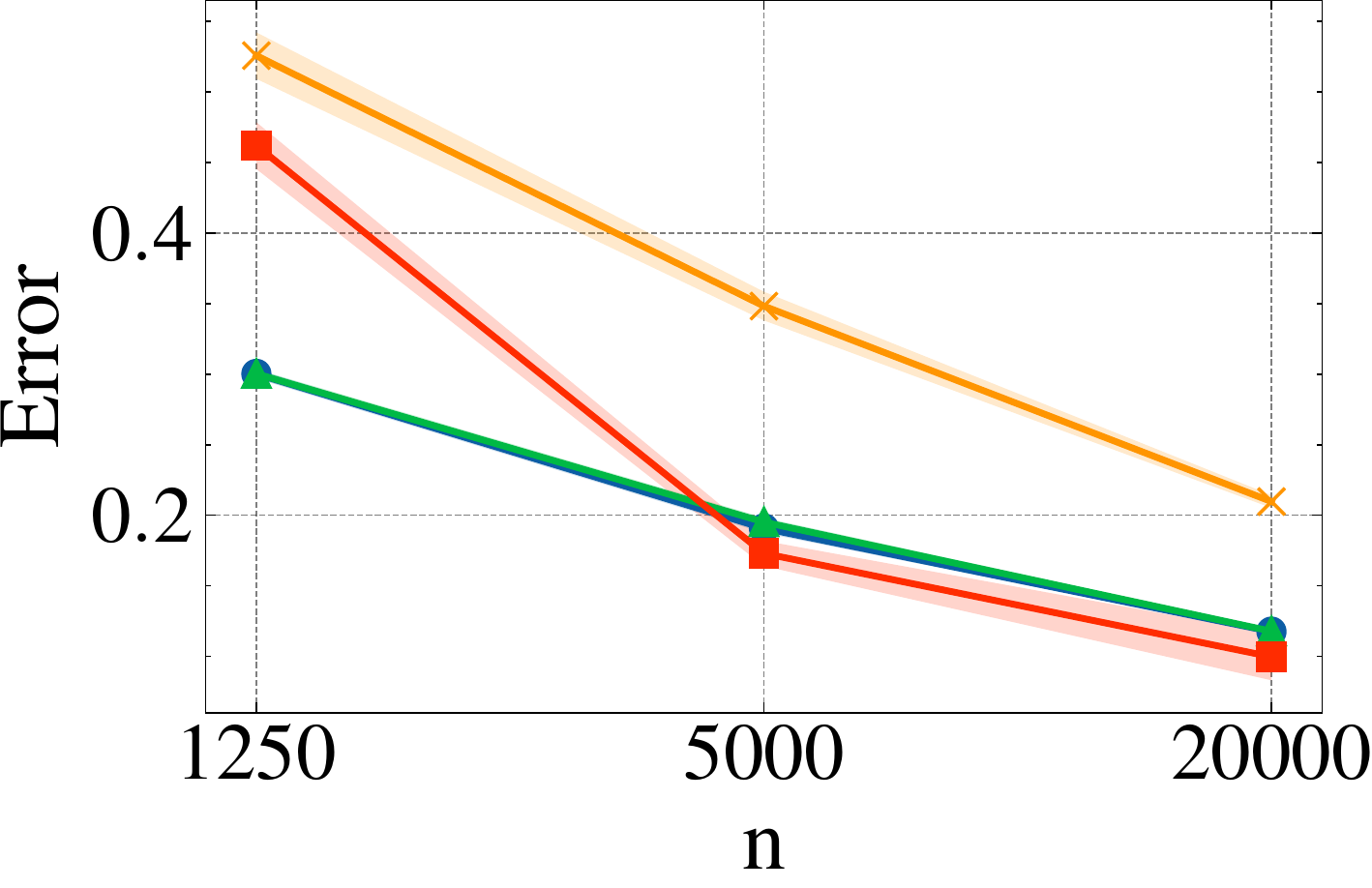}
        \caption{SVHN}
        \label{fig:gen_svhn}
    \end{subfigure}
    \hfill
    \begin{subfigure}[b]{0.3\textwidth}
        \centering
        \includegraphics[width=\textwidth]{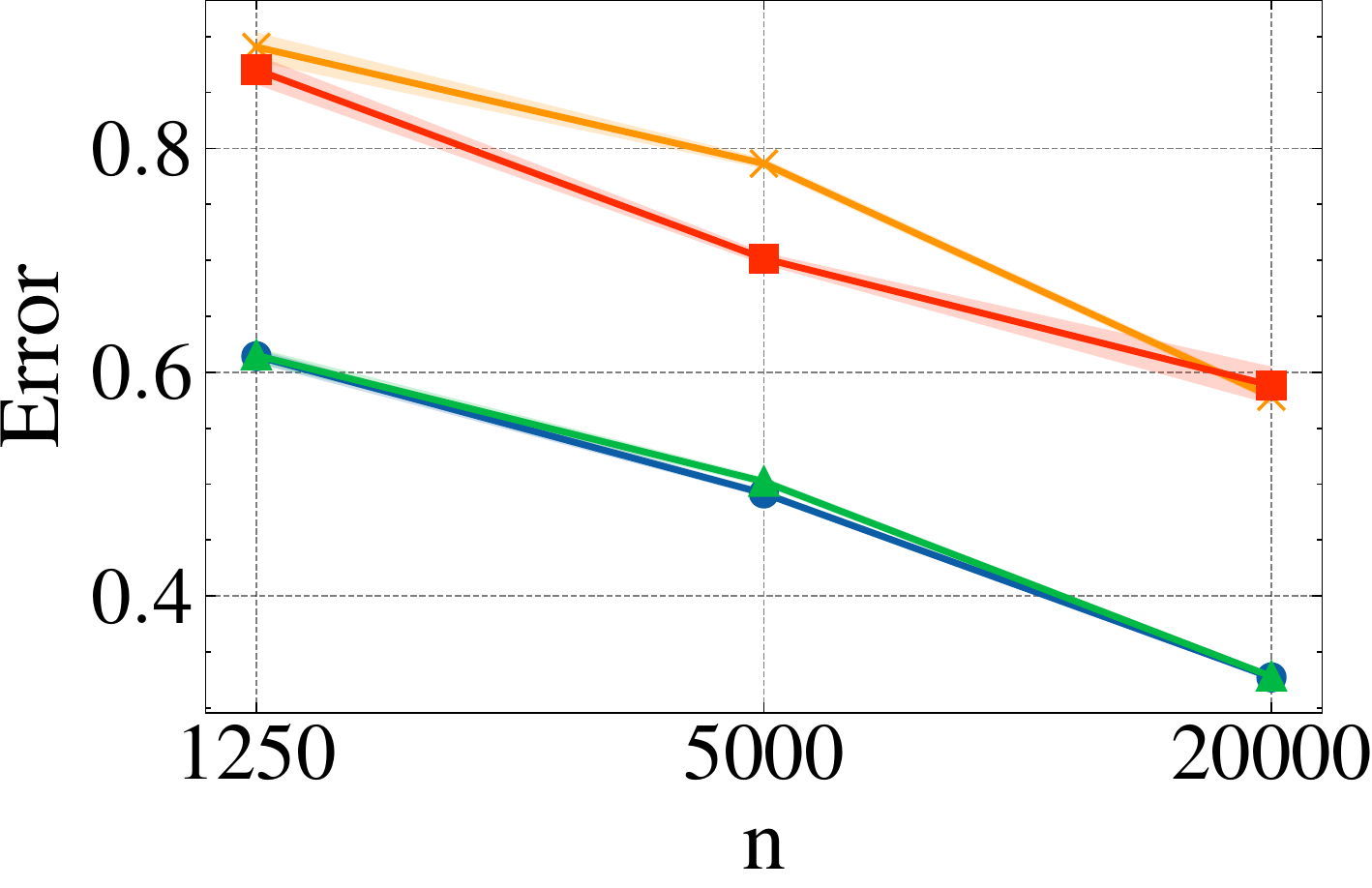}
        \caption{CIFAR10}
        \label{fig:gen_cifar}
    \end{subfigure}
    \caption{Comparison of expected generalization gap, mean test error, \textit{f}-CMI bound and recognizability of representations on three datasets with varying training set sizes.}
    \label{fig:comp_gap}
\end{figure*}

\begin{figure*}[t]
    \centering
    \begin{minipage}[c]{0.67\textwidth}
      \begin{subfloat}[Fashion] {
          \label{fig:recog_fashion}
          \includegraphics[width=0.31\linewidth]{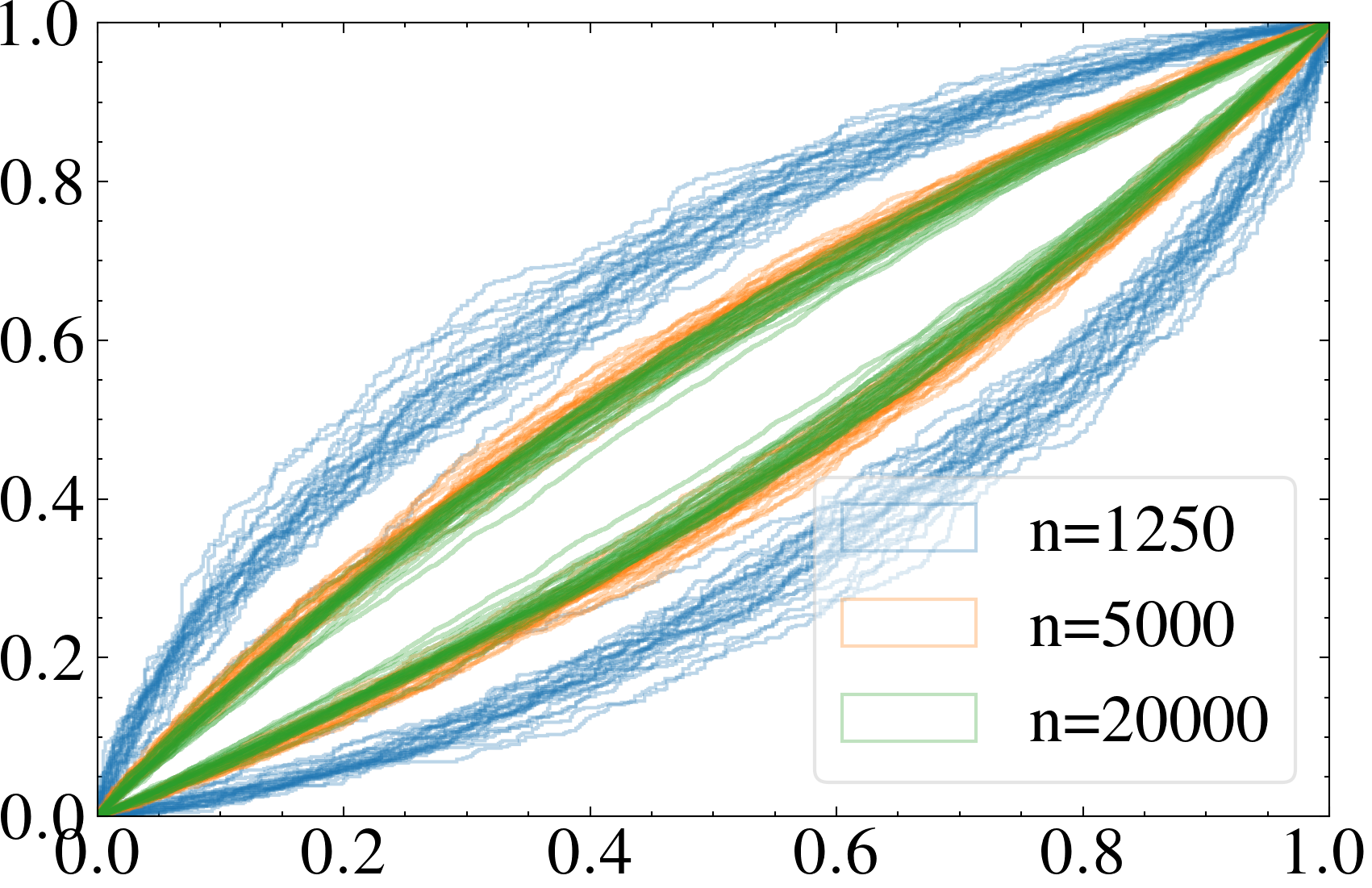}}
      \end{subfloat}
      \hfill
      \begin{subfloat}[SVHN] {
          \label{fig:recog_svhn}
          \includegraphics[width=0.31\linewidth]{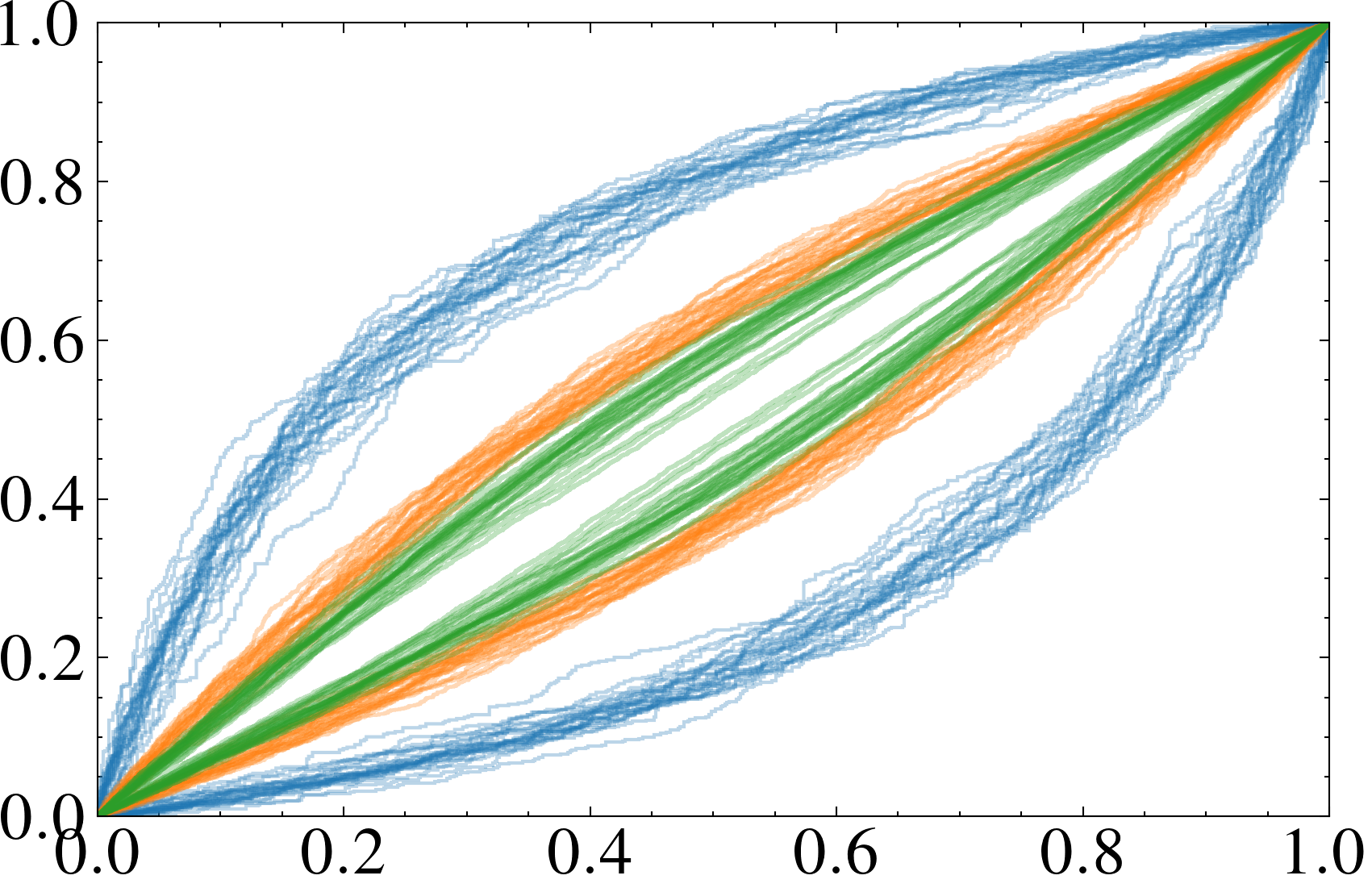}}
      \end{subfloat}
      \hfill
      \begin{subfloat}[CIFAR10] {
          \label{fig:recog_cifar}
          \includegraphics[width=0.31\linewidth]{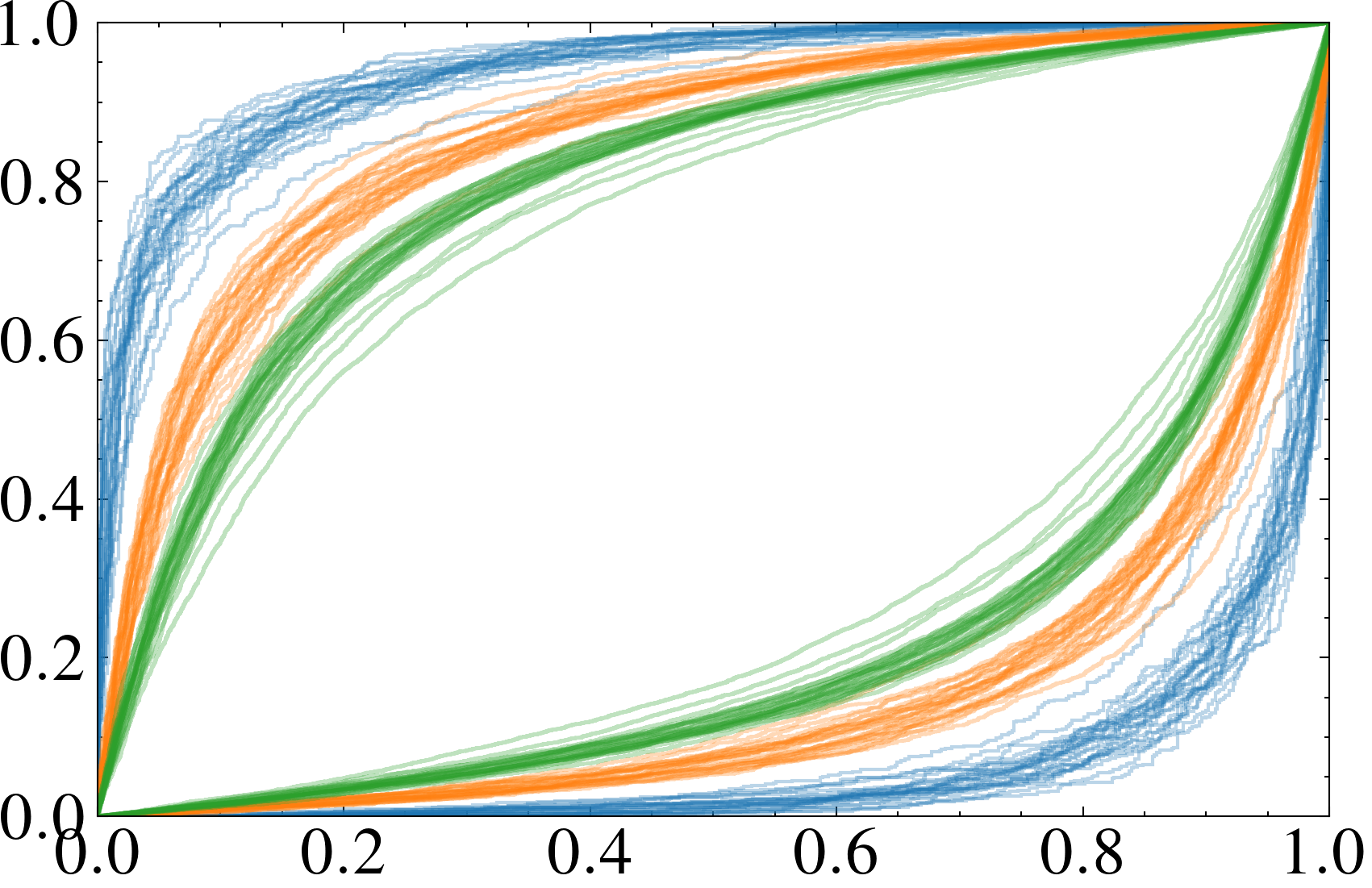}}
      \end{subfloat}
      \caption{Visualization of the achievable region corresponding to the recognizability of representations on three datasets with varying training set sizes. Each polygon represents one run of the experiment. Smaller polygons indicate less recognizability. \label{fig:recog_vis}}
    \end{minipage}
    \hfill
    \begin{minipage}[c]{0.28\textwidth}
      \centering
      \includegraphics[width=\linewidth]{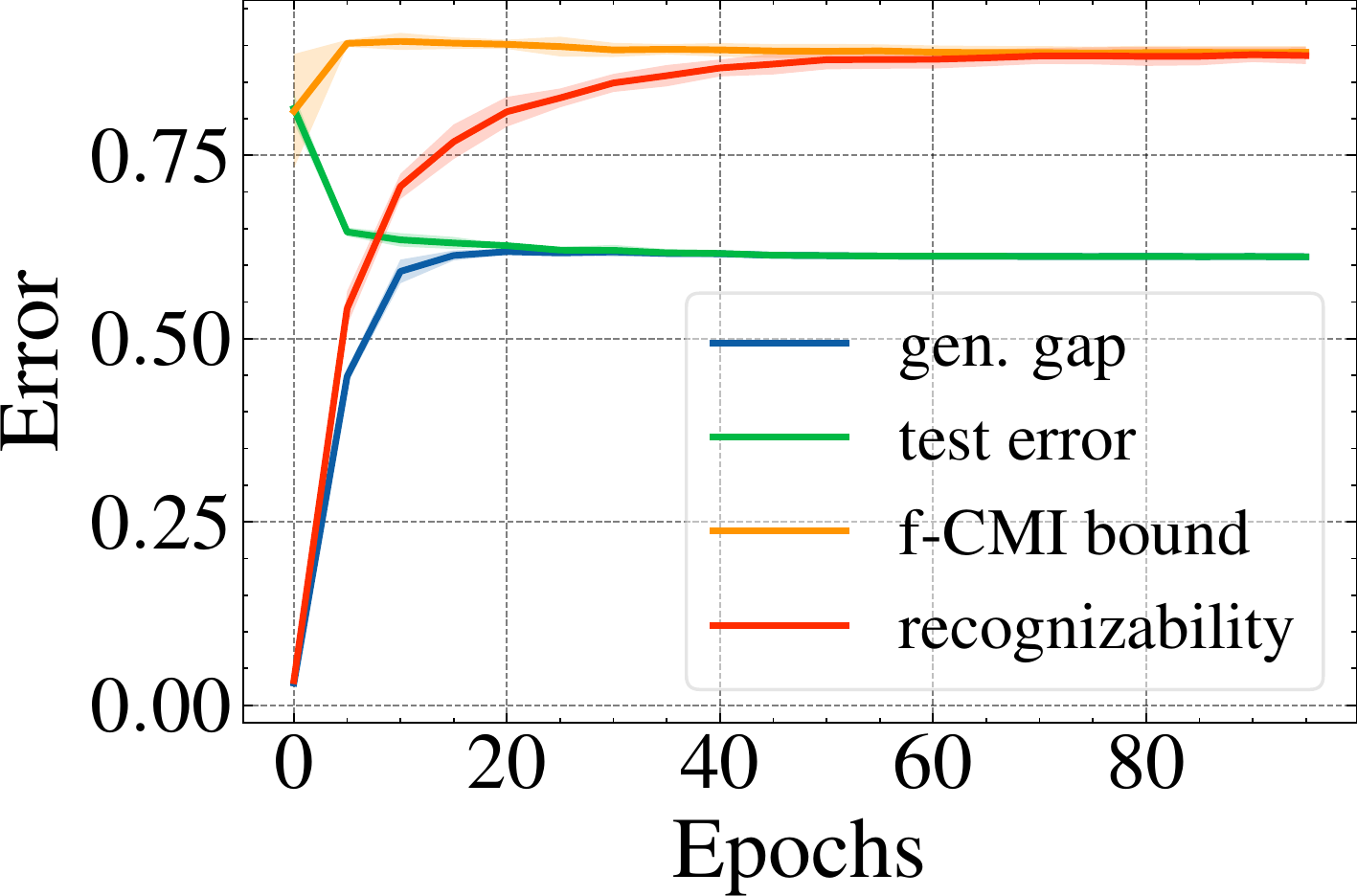}
      \caption{Dynamics of recognizability on CIFAR10 with size $n=1250$.}
      \label{fig:cifar10_1250_epoch}
    \end{minipage}
  \end{figure*}

\phantomsection
\subsubsection{Impacts of \texorpdfstring{$\beta$}{beta}}
To investigate the impacts of the regularization term of the RIB objective, we train models using the RIB objective with various $\beta$. The results on SVHN and CIFAR10 are illustrated in \cref{fig:beta}. We can observe that: \textit{i}) When training data is scarce, the model obtains the best performance at larger $\beta$. This is expected since the model tends to over-fit to the data in this case.
% , and it again demonstrates that the RIB objective can regularize the model to obtain more generalized representations;
\textit{ii}) SVHN is more tolerant of large $\beta$ compared to CIFAR10. When increasing the $\beta$ to more than ${10}^3$, the model still outperforms the CE baseline in most cases, while too large $\beta$ may lead to instability in CIFAR10 training. We consider this because the features of SVHN are easier to learn than those of CIFAR10 and are more likely to be over-fitted to irrelevant features, thus regularization brings greater improvement. This provides us with some indications for choosing the best $\beta$ based on the characteristics of the dataset.

\subsection{Estimation of Generalization Gap}
To verify the recognizability of representations can be used to indicate the generalization performance, we conduct experiments on different training set sizes $n\in \{1250,5000,20000\}$ with the smaller dataset being a subset of the larger dataset. We first train the models on different dataset sizes respectively, and then train the recognizability critics using the representations obtained from the trained models on the smallest data subset respectively. This allows us to evaluate the generalization to the unseen samples outside the smallest subset.
Each experiment is repeated $k_1 k_2$ times to evaluate the \textit{f}-CMI bound. As illustrated in \cref{fig:comp_gap}, we find that \textit{i}) both the recognizability of representations and the \textit{f}-CMI bound reflect the generalization to the unseen data, and the recognizability of representations even obtains better estimates in some cases; \textit{ii}) since the recognizability of representations is a parameterized approximation rather than an upper bound, the risk may sometimes be underestimated as in \cref{fig:gen_svhn}.
We also visualize the achievable region corresponding to the recognizability of representations, which is obtained by plotting the ROC curve of the score function and its central symmetry about the point (1/2, 1/2). The symmetry of the achievable region is due to the fact that we can construct an opposite test when any test chooses $H_0$ it chooses $H_1$ and vice versa.
The visualization results are illustrated in \cref{fig:recog_vis}.
We find that the better the generalization performance, the smaller the area of achievable regions, which validates the plausibility of our formulation. Since CIFAR10 features are more difficult to learn thus easier to over-fit, it can be seen that a larger area is covered than the other two datasets.

\phantomsection
\subsubsection{Recognizability Dynamics}
To investigate how the recognizability changes during training, we evaluate the recognizability every 5 epochs and plot the dynamics of recognizability as the number of training epochs changes in \cref{fig:cifar10_1250_epoch}.
We observe that recognizability can track the generalization gap as expected and eventually converge to around the \textit{f}-CMI bound. See Appendix C for more results.

\section{Conclusions and Future Work}
In this paper, we establish the connection between the recognizability of representations and the \textit{f}-CMI bound, and propose the RIB to regularize the recognizability of representations through Bregman divergence. 
We conducted extensive experiments on several commonly used datasets to demonstrate the effectiveness of RIB in regularizing the model and estimating the generalization gap. 
% The proposed RIB has the potential to be applied to a wider range of tasks. For instance, we found that RIB can improve the robustness of the model to label noise, especially in small data scale and relatively low noise scenarios, although it is not specifically designed for this case (see Appendix C). 
Future work can investigate whether RIB is also effective on other tasks such as few-shot learning, out-of-distribution detection, domain generalization, etc.

\appendix

\section{Proof of Theorem 1}
\label{appx:proof1}
% \begin{theorem}
%     Let $\tilde{z}$ be a realization of supersample $\tilde{Z}$, $T$ be the representation variable defined as above and $U \sim \mathrm{Unif}(\{0,1\}^n)$. The recognizability of representations $\Re(P_{T|\tilde{Z}=\tilde{z}},P_{T|U,\tilde{Z}=\tilde{z}})$ is not larger than $1.443 \cdot I(T;U|\tilde{z})$.
% \end{theorem}

We start with giving a formal definition of the ROC curve for log-likelihood ratio test (LRT).
We consider two hypotheses, $H_0$ and $H_1$, corresponding to the detection of two distributions $P_{T|\tilde{Z}=\tilde{z}}$ and $P_{T|U,\tilde{Z}=\tilde{z}}$, respectively,
$$ H_0: T \sim P_{T|\tilde{Z}=\tilde{z}} \qquad H_1: T \sim P_{T|U,\tilde{Z}=\tilde{z}}\,. $$
Denote $\xi(T)$ as the LRT statistic random variable w.r.t. $\xi$, while LRT is obtained by
$$ \xi(t) = \log \frac{P_T(t|H_1)}{P_T(t|H_0)} = \log \frac{P_{T|U,\tilde{Z}=\tilde{z}}(t)}{P_{T|\tilde{Z}=\tilde{z}}(t)},
$$
where $P_T(t|H_0)$ and $P_T(t|H_1)$ denote the distributions of $T$ under the hypothesis $H_0$ $H_1$, respectively.
We also define $P_0(\tau)=\mathrm{Pr}(\xi(T) \ge \tau | H_0)$ and $P_1(\tau)=\mathrm{Pr}(\xi(T) \ge \tau | H_1)$, where $\tau$ is a threshold. Then, each point on the ROC curve can be achieved by $(P_0(\tau), P_1(\tau))$.
Let $x=P_0(\tau)$, $P_{0}^{-1}$ be the inverse function of $P_{0}$, the ROC curve can be expressed as $\psi_\xi(x)=P_1(P_{0}^{-1}(x))$. For the sake of brevity, we omit the subscript $\xi$. Note that due to the properties of ROC, $\psi(x)$ is twice continuously differentiable on the interval $[0,1]$ and should satisfy the following conditions \cite{Levy_2008,Khajavi_Kuh_2018}:
\begin{align}
    \int_{0}^{1} \psi^\prime(x)\,dx = 1   \label{eq:cond1}\tag{C1} \\
    \psi^\prime(x) \ge 0  \label{eq:cond2}\tag{C2}                 \\
    \psi^{\prime\prime}(x) \le 0  \label{eq:cond3}\tag{C3}
\end{align}
We can express the area under the ROC curve by
\begin{equation*}
    \mathrm{AUCROC}_\xi = \int_{0}^{1} \psi(x) \, dx \,.
\end{equation*}
Our proof relies on the following lemma.
\begin{lemma}[Lem. 1 of \cite{Khajavi_Kuh_2018}]
    Let $\xi_0 := \xi(T)|H_0$, $\xi_1 := \xi(T)|H_1$, $P_{\xi_0}(t)$ and $P_{\xi_1}(t)$ be the distribution of the random variables $\xi_0$ and $\xi_1$, respectively. Given the ROC curve, $\psi(x)$, we can compute the Kullback-Leibler (KL) divergence
    \begin{equation}
        \label{eq:lemma1}
        D_\mathrm{KL}(P_{\xi_1}(t) \Vert P_{\xi_0}(t)) = -\int_{0}^{1} \log \psi^\prime(x) \, dx \,.
    \end{equation}
\end{lemma}
This lemma establishes the relationship between receiver operating characteristic (ROC) curves and KL divergence. We will use it to connect the recognizability of representations and the KL divergence.

\begin{proof}[Proof of Theorem 1]
    We have that
    \begin{align*}
        D_\mathrm{KL}(P_{\xi_1}(t) \Vert P_{\xi_0}(t))
         & \le D_\mathrm{KL}(P_{T}(t|H_1) \Vert P_{T}(t|H_0))                                 \\
         & =D_\mathrm{KL}(P_{T|U,\tilde{Z}=\tilde{z}}(t) \Vert P_{T|\tilde{Z}=\tilde{z}}(t)),
    \end{align*}
    where the inequality holds by the data processing property for divergences \cite{Polyanskiy_Wu} and the equation holds by definition.
    Now we can prove that the theorem holds by proving the following inequality
    \begin{equation}
        \label{ineq:goal1}
        \Re(P_{T|\tilde{Z}=\tilde{z}},P_{T|U,\tilde{Z}=\tilde{z}}) \le D_\mathrm{KL}(P_{\xi_1}(t) \Vert P_{\xi_0}(t)) + \log \frac{e}{2} \,.
    \end{equation}
    By definition, we have that
    \begin{equation}
        \label{eq:recog_def}
        \Re(P_{T|\tilde{Z}=\tilde{z}},P_{T|U,\tilde{Z}=\tilde{z}}) = 2 \int_{0}^{1} \psi(x)\,dx - 1 \,.
    \end{equation}
    Using integration by parts, we have $\int_{0}^{1} \psi(x)\,dx = 1 - \int_{0}^{1} x\psi^\prime(x)\,dx$. Substituting this into \cref{eq:recog_def} gives
    \begin{equation}
        \label{eq:recog_def2}
        \Re(P_{T|\tilde{Z}=\tilde{z}},P_{T|U,\tilde{Z}=\tilde{z}}) = 1 - 2 \int_{0}^{1} x\psi^\prime(x)\,dx \,.
    \end{equation}
    Substituting \cref{eq:lemma1,eq:recog_def2} into Inequality (\ref{ineq:goal1}) and rearranging terms yields the following new goal:
    \begin{equation}
        \label{ineq:goal2}
        \int_{0}^{1} 2x\psi^\prime(x) - \log \psi^\prime(x) \, dx \ge \log 2 \,.
    \end{equation}
    Next we will prove that the minimum of the LHS of Inequality (\ref{ineq:goal2}) is greater than $\log 2$. We convert the proof into an optimization problem as follows
    \begin{equation}
        \label{eq:optimization}
        \begin{gathered}
            \arg \min_{\psi^\prime(x)}
            \int_{0}^{1} 2x\psi^\prime(x) - \log \psi^\prime(x) \, dx \\
            \text{s.t. (\ref{eq:cond1})-(\ref{eq:cond3})} \,.
        \end{gathered}
    \end{equation}
    We then express \cref{eq:optimization} with a Lagrangian multiplier $\lambda > 0$ and yields the following Lagrangian
    \begin{equation*}
        \label{eq:lagrangian}
        L(\psi^\prime(x), \lambda) = \int_{0}^{1} 2x\psi^\prime(x) - \log \psi^\prime(x) \, dx + \lambda (\int_{0}^{1} \psi^\prime(x)\,dx - 1) \,.
    \end{equation*}
    Since the Lagrangian $L(\psi^\prime(x), \lambda)$ is a convex functional of $\psi^\prime(x)$, we can now calculate its derivative w.r.t. $\psi^\prime(x)$ by
    \begin{equation*}
        \label{eq:derivative}
        \frac{\partial L(\psi^\prime(x), \lambda)}{\partial \psi^\prime(x)} =
        \int_{0}^{1} (2x - \frac{1}{\psi^\prime(x)} + \lambda) \,dx \,.
    \end{equation*}
    Setting $\frac{\partial L(\psi^\prime(x), \lambda)}{\partial \psi^\prime(x)} = 0$, we obtain
    \begin{equation}
        \label{eq:hprimez}
        \psi^\prime(x) = \frac{1}{2x+\lambda} \,.
    \end{equation}
    Substituting \cref{eq:hprimez} into Inequality (\ref{ineq:goal2}) and calculating the definite integral gives
    \begin{equation*}
        \label{ineq:goal_left}
        \int_{0}^{1} \frac{2x}{2x+\lambda} + \log (2x+\lambda) \, dx = \log (2+\lambda) \ge \log 2,
    \end{equation*}
    which concludes the proof of the theorem.
\end{proof}

\section{Additional Experimental Details}
\label{appx:details}
\subsection{Network Architectures}
In this work, we employ two deep models: a neural network encoder to learn the representations and a recognizability critic to learn the recognizability of representations.
%  For the encoder, we adopt the architecture similar to \cite{DBLP:conf/nips/HarutyunyanRSG21}. 
The specific network architectures are presented in \cref{table:arch_enc,table:arch_recog}, respectively.

\begin{table}[ht]
    \centering
    \begin{tabular}{l|p{0.75\columnwidth}}
        \toprule
        Type & Parameters                                                                                  \\
        \midrule
        Conv & 128 filters, $4 \times 4$ kernels, stride 2, padding 1, \newline batch normalization, ReLU  \\
        Conv & 128 filters, $4 \times 4$ kernels, stride 2, padding 1, \newline batch normalization, ReLU  \\
        Conv & 256 filters, $3 \times 3$ kernels, stride 2, padding 0, \newline batch normalization, ReLU  \\
        Conv & 1024 filters, $3 \times 3$ kernels, stride 1, padding 0, \newline batch normalization, ReLU \\
        FC   & 1024 units, ReLU                                                                            \\
        FC   & 512 units, ReLU                                                                             \\
        \bottomrule
    \end{tabular}
    \caption{The network architecture of the neural network encoder.}
    \label{table:arch_enc}
\end{table}

\begin{table}[ht]
    \centering
    \begin{tabular}{l|l}
        \toprule
        Type & Parameters                 \\
        \midrule
        FC   & 1024 units, LeakyReLU      \\
        FC   & 1024 units, LeakyReLU      \\
        FC   & 1024 units, LeakyReLU      \\
        FC   & 1 unit,  linear activation \\
        \bottomrule
    \end{tabular}
    \caption{The network architecture of the recognizability critic.}
    \label{table:arch_recog}
\end{table}

\subsection{Training Set Sizes}
We set three data scales analogous to \cite{DBLP:conf/nips/HarutyunyanRSG21} but with some minor changes. For MNIST, Fashion-MNIST, SVHN and CIFAR10, we set the training set sizes to $\{1250,5000,20000\}$, the smaller one is a quarter of the larger one; for STL10, we set the training set sizes to $\{625,1250,2500\}$, since it contains only 5000 training samples and we should keep at least half of them to form a ghost set.

\section{Additional Experimental Results}
\label{appx:results}
\subsection{Density-ratio Matching Using Different Measures}
The Bregman divergence is a generalized measure of the difference between two points. Applying different $F(R)$ will degrade the Bregman divergence to correspondingly different measures. Several choices of $F(R)$ commonly used in density-ratio matching \cite{Sugiyama_Suzuki_Kanamori_2012} are given in \cref{table:density_ratio}. \cref{table:bregman} shows that they do not exhibit significant performance differences.
\begin{table*}[ht]
    \centering
    \begin{tabular}{lll}
        \toprule
        Name                             & $F(R)$                            & $D_{\mathrm{BR}}(R^{*} \Vert R)$                         \\
        \midrule
        Binary KL (BKL) divergence       & $ R \log R - (1 + R) \log(1 + R)$ & $ (1+R) \log \frac{1+R}{1+R^*} + R^* \log \frac{R^*}{R}$ \\
        Squared (SQ) distance            & $(R - 1)^2 / 2$                   & $ \frac{1}{2}(R^*-R)^2$                                  \\
        Unnormalized KL (UKL) divergence & $R \log R - R$                    & $R^* \log \frac{R^*}{R} - R^* + R$                       \\
        \bottomrule
    \end{tabular}
    \caption{Summary of density ratio matching using different measures.}
    \label{table:density_ratio}
\end{table*}

\begin{table*}[ht]
    \centering
    \begin{tabular}{lcccccccccccccccc}
        \toprule
        \multirow{2}{*}{Method} & \multicolumn{3}{c}{Fashion}              & \multicolumn{3}{c}{SVHN}                 & \multicolumn{3}{c}{CIFAR10}                                                                                                                                                                                                                                                                               \\
        \cmidrule(lll){2-4} \cmidrule(lll){5-7} \cmidrule(lll){8-10}
        {}                      & 1250                                     & 5000                                     & 20000                                   & 1250                                     & 5000                                     & 20000                                    & 1250                                     & 5000                                     & 20000                                    \\
        \midrule
        RIB                     & \textbf{17.05{\scriptsize \textpm 0.41}} & 13.57{\scriptsize \textpm 0.28}          & \textbf{9.32{\scriptsize \textpm 0.25}} & \textbf{27.72{\scriptsize \textpm 1.17}} & 17.37{\scriptsize \textpm 0.52}          & 10.80{\scriptsize \textpm 0.27}          & \textbf{59.16{\scriptsize \textpm 1.17}} & 47.57{\scriptsize \textpm 0.60}          & \textbf{31.06{\scriptsize \textpm 0.38}} \\
        RIB\textsubscript{SQ}   & 17.06{\scriptsize \textpm 0.51}          & 13.63{\scriptsize \textpm 0.34}          & 9.46{\scriptsize \textpm 0.24}          & 28.26{\scriptsize \textpm 1.05}          & \textbf{17.19{\scriptsize \textpm 0.44}} & \textbf{10.74{\scriptsize \textpm 0.23}} & 59.46{\scriptsize \textpm 1.03}          & 47.20{\scriptsize \textpm 1.07}          & 31.33{\scriptsize \textpm 0.41}          \\
        RIB\textsubscript{UKL}  & 17.06{\scriptsize \textpm 0.39}          & \textbf{13.51{\scriptsize \textpm 0.39}} & 9.40{\scriptsize \textpm 0.27}          & 28.14{\scriptsize \textpm 1.27}          & 17.20{\scriptsize \textpm 0.38}          & 10.80{\scriptsize \textpm 0.32}          & 59.57{\scriptsize \textpm 1.02}          & \textbf{47.15{\scriptsize \textpm 0.60}} & 31.13{\scriptsize \textpm 0.50}          \\
        \bottomrule
    \end{tabular}
    \caption{Comparison of the mean test error (\%) on three datasets with different measure function as Bregman divergence.}
    \label{table:bregman}
\end{table*}

% \subsection{Comparing with other IB variants}

\subsection{Regularization Effects on MNIST and STL10}
We report additional results on MNIST and STL10 that are not included in the main text due to space constraints.
\cref{table:perf2} shows that RIB still outperforms other baselines across all datasets and all training set sizes. This again demonstrates that the effectiveness of the RIB objective in improving the generalization performance of the model.

\begin{table*}[ht]
    \centering
    \begin{tabular}{lcccccccccccccccc}
        \toprule
        \multirow{2}{*}{Method} & \multicolumn{3}{c}{MNIST}               & \multicolumn{3}{c}{STL10}                                                                                                                                                                                          \\
        \cmidrule(lll){2-4} \cmidrule(lll){5-7}
        {}                      & 1250                                    & 5000                                    & 20000                                   & 625                                      & 1250                                     & 2500                                     \\
        \midrule
        CE                      & 4.75{\scriptsize \textpm 1.00}          & 2.43{\scriptsize \textpm 0.29}          & 1.00{\scriptsize \textpm 0.09}          & 66.62{\scriptsize \textpm 1.20}          & 61.88{\scriptsize \textpm 1.01}          & 56.21{\scriptsize \textpm 0.97}          \\
        L2                      & 4.25{\scriptsize \textpm 0.39}          & 2.46{\scriptsize \textpm 0.26}          & 0.95{\scriptsize \textpm 0.19}          & 66.47{\scriptsize \textpm 0.71}          & 62.26{\scriptsize \textpm 1.50}          & 56.77{\scriptsize \textpm 1.38}          \\
        Dropout                 & 4.58{\scriptsize \textpm 0.40}          & 2.47{\scriptsize \textpm 0.21}          & 1.01{\scriptsize \textpm 0.13}          & 66.58{\scriptsize \textpm 0.98}          & 61.77{\scriptsize \textpm 0.86}          & 56.53{\scriptsize \textpm 0.92}          \\
        VIB                     & 3.64{\scriptsize \textpm 0.20}          & 1.98{\scriptsize \textpm 0.13}          & 0.87{\scriptsize \textpm 0.07}          & 63.69{\scriptsize \textpm 1.11}          & 59.76{\scriptsize \textpm 0.95}          & 54.73{\scriptsize \textpm 0.67}          \\
        NIB                     & 4.29{\scriptsize \textpm 0.27}          & 2.34{\scriptsize \textpm 0.18}          & 0.95{\scriptsize \textpm 0.07}          & 66.16{\scriptsize \textpm 0.72}          & 61.49{\scriptsize \textpm 0.99}          & 55.71{\scriptsize \textpm 0.64}          \\
        DIB                     & 3.86{\scriptsize \textpm 0.45}          & 1.97{\scriptsize \textpm 0.12}          & 0.86{\scriptsize \textpm 0.06}          & 66.20{\scriptsize \textpm 0.75}          & 61.43{\scriptsize \textpm 0.94}          & 55.63{\scriptsize \textpm 0.71}          \\
        PIB                     & 4.35{\scriptsize \textpm 0.24}          & 2.29{\scriptsize \textpm 0.17}          & 0.96{\scriptsize \textpm 0.08}          & 66.16{\scriptsize \textpm 0.97}          & 61.59{\scriptsize \textpm 0.63}          & 55.93{\scriptsize \textpm 0.53}          \\
        RIB                     & \textbf{3.47{\scriptsize \textpm 0.29}} & \textbf{1.82{\scriptsize \textpm 0.12}} & \textbf{0.85{\scriptsize \textpm 0.08}} & \textbf{62.88{\scriptsize \textpm 0.81}} & \textbf{58.99{\scriptsize \textpm 1.04}} & \textbf{54.34{\scriptsize \textpm 0.77}} \\
        \bottomrule
    \end{tabular}
    \caption{Comparison of the mean test error (\%) on MNIST and STL10 with different training set sizes.}
    \label{table:perf2}
\end{table*}

\subsection{Estimation of Generalization Gap on MNIST and STL10}
We verify that the recognizability of representations is able to indicate the generalization to unseen data on MNIST and STL10. As illustrated in \cref{fig:comp_gap2}, the recognizability curves are still able to keep track of the generalization gaps as well as the f-CMI bounds.
The visualization results of the achievable region corresponding to the recognizability of representations are illustrated in \cref{fig:recog_vis2}.
We can still find that the better the generalization performance, the smaller the area of achievable regions, and the area corresponding to MNIST is significantly smaller than that of STL10 since the former is much easier to learn, while the latter has fewer samples making the model prone to over-fitting.

% \begin{figure}
%     \centering
%     \includegraphics[width=\columnwidth]{fig/svhn_loss_curve.pdf}
%     \caption{Mean test risk curves on SVHN with training set sizes of 1250 and 20000, respectively.  ``RIB-adv'' represents optimizing the RIB objective with adversarial training.}
%     \label{fig:risk_curve_svhn}
% \end{figure}

\begin{figure}
    \centering
    \begin{subfigure}[b]{0.75\columnwidth}
        \centering
        \includegraphics[width=\textwidth]{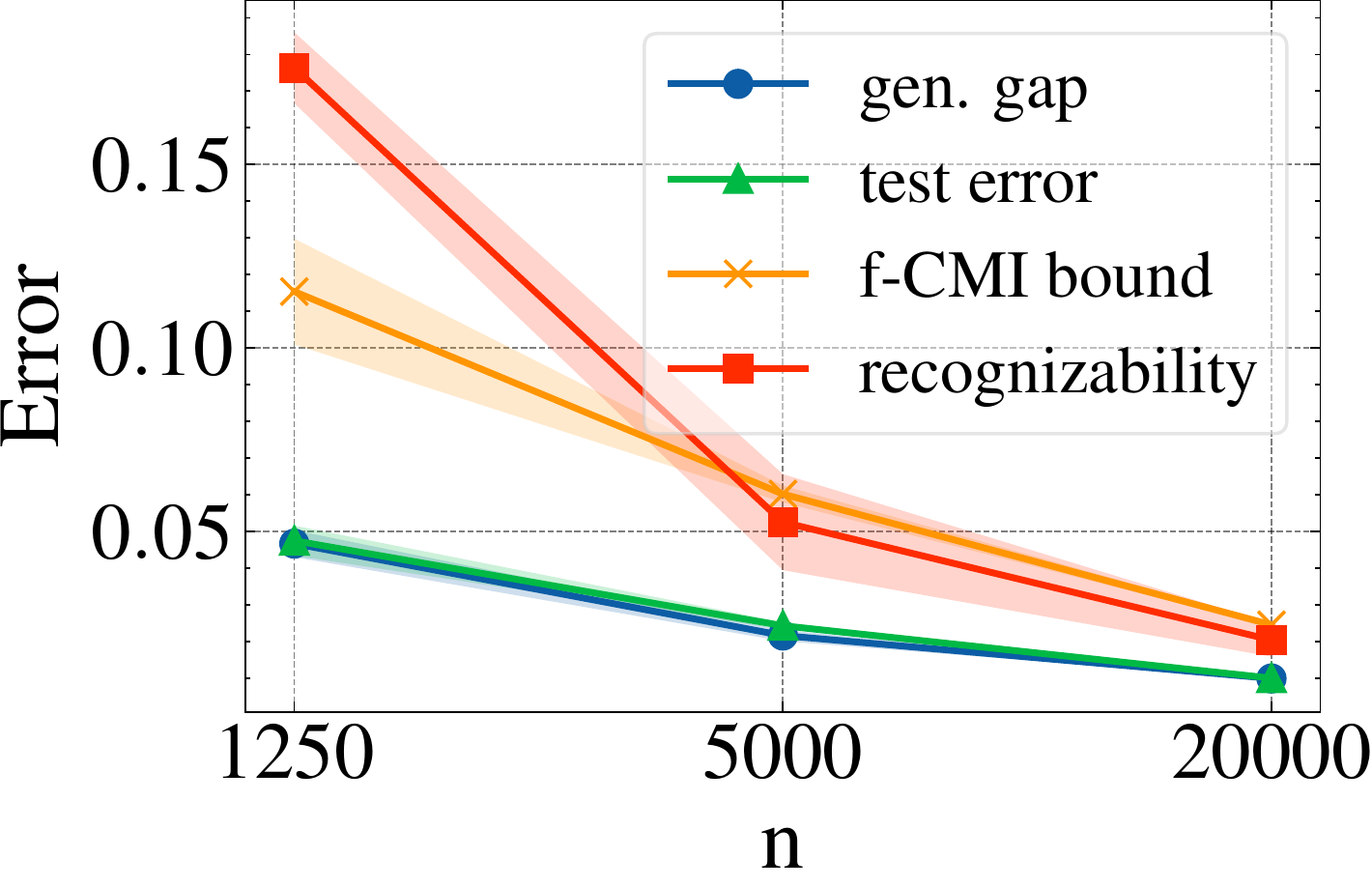}
        \caption{MNIST}
        \label{fig:gen_mnist}
    \end{subfigure}
    \begin{subfigure}[b]{0.75\columnwidth}
        \centering
        \includegraphics[width=\textwidth]{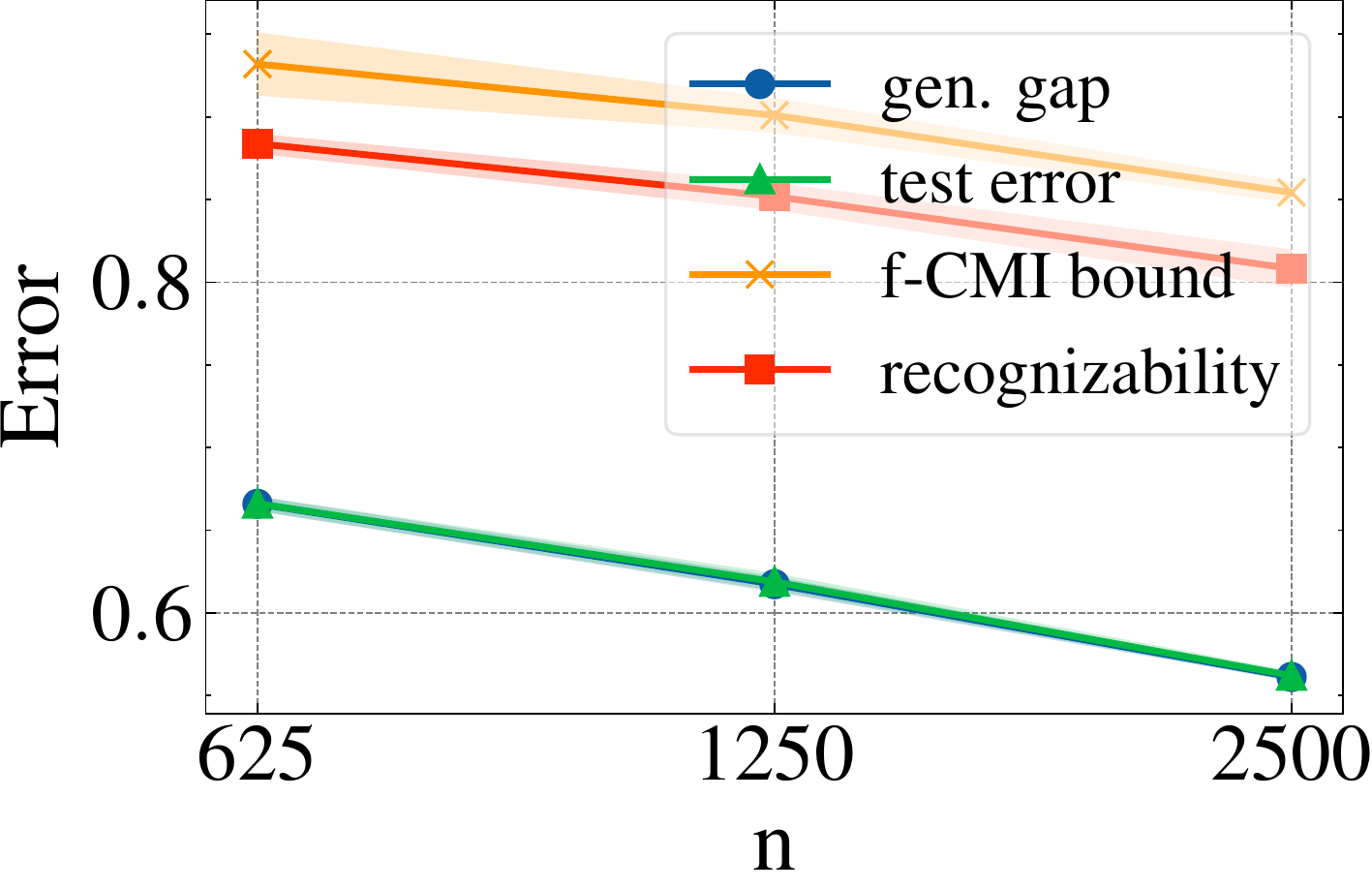}
        \caption{STL10}
        \label{fig:gen_stl10}
    \end{subfigure}
    \caption{Comparison of expected generalization gap, mean test error, f-CMI bound and recognizability of representations on MNIST and STL10 with varying training set sizes.}
    \label{fig:comp_gap2}
\end{figure}

\begin{figure}
    \centering
    \begin{subfigure}[b]{0.3\textwidth}
        \centering
        \includegraphics[width=\textwidth]{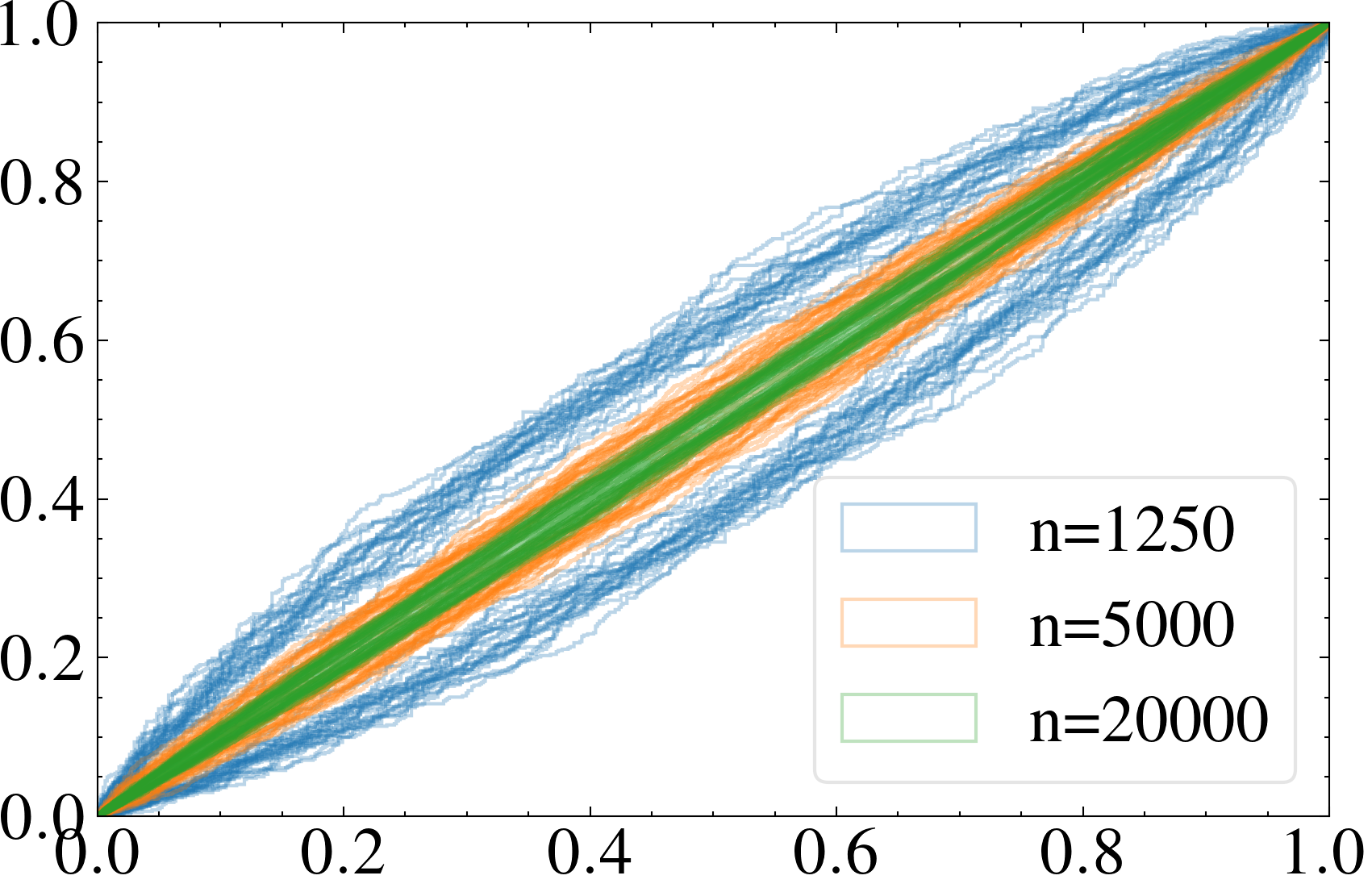}
        \caption{MNIST}
        \label{fig:recog_mnist}
    \end{subfigure}
    \begin{subfigure}[b]{0.3\textwidth}
        \centering
        \includegraphics[width=\textwidth]{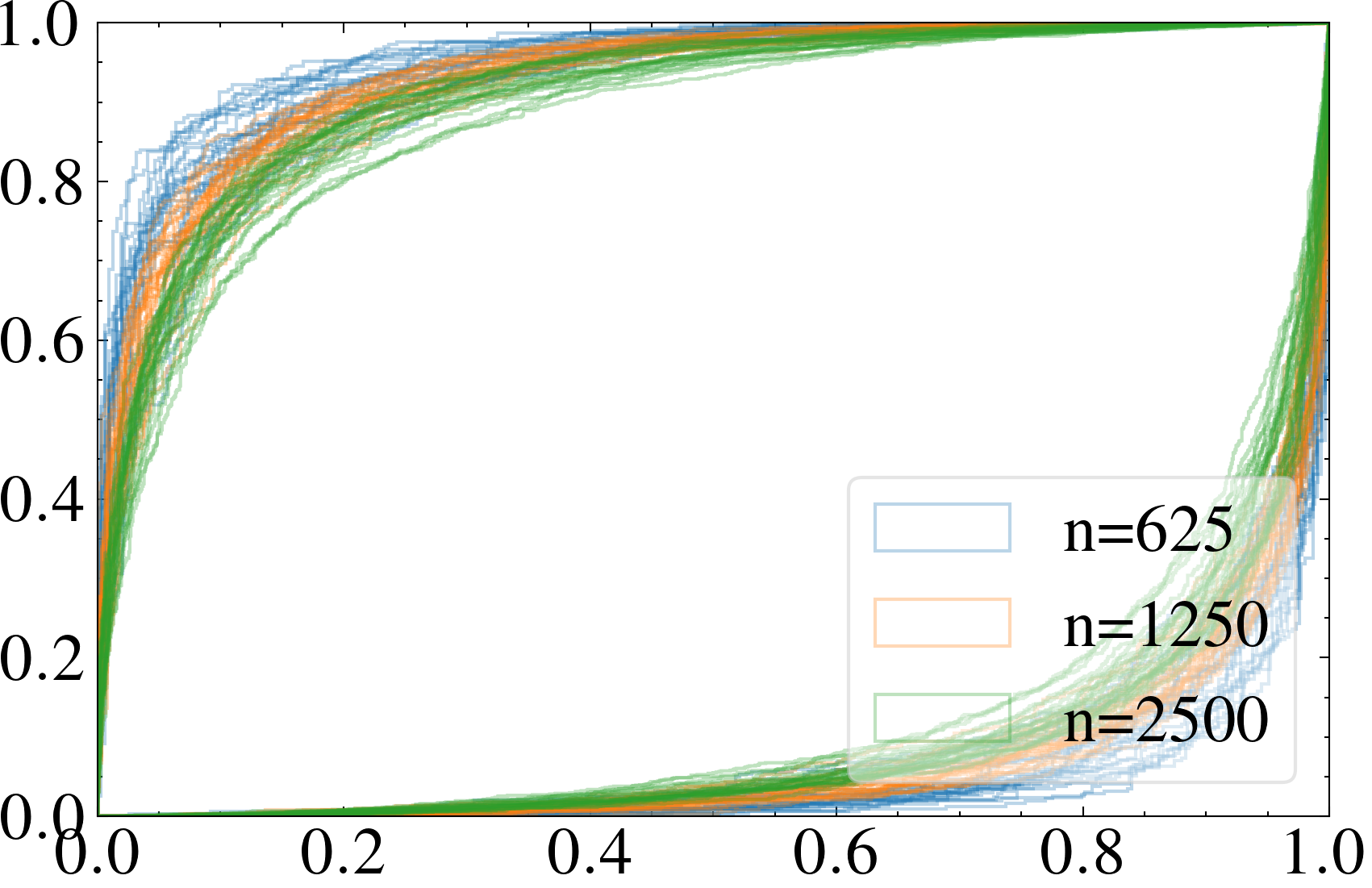}
        \caption{STL10}
        \label{fig:recog_stl10}
    \end{subfigure}
    \caption{Visualization of the achievable region corresponding to the recognizability of representations on MNIST and STL10 with varying training set sizes. Each polygon represents one run of the experiment. Smaller polygons indicate less recognizability.}
    \label{fig:recog_vis2}
\end{figure}

\subsection{Recognizability Dynamics}
We report the dynamics of recognizability on CIFAR10 with training set sizes of 5000 and 20000 in \cref{fig:epoch}.

\begin{figure}
    \centering
    \begin{subfigure}[b]{0.75\columnwidth}
        \centering
        \includegraphics[width=\textwidth]{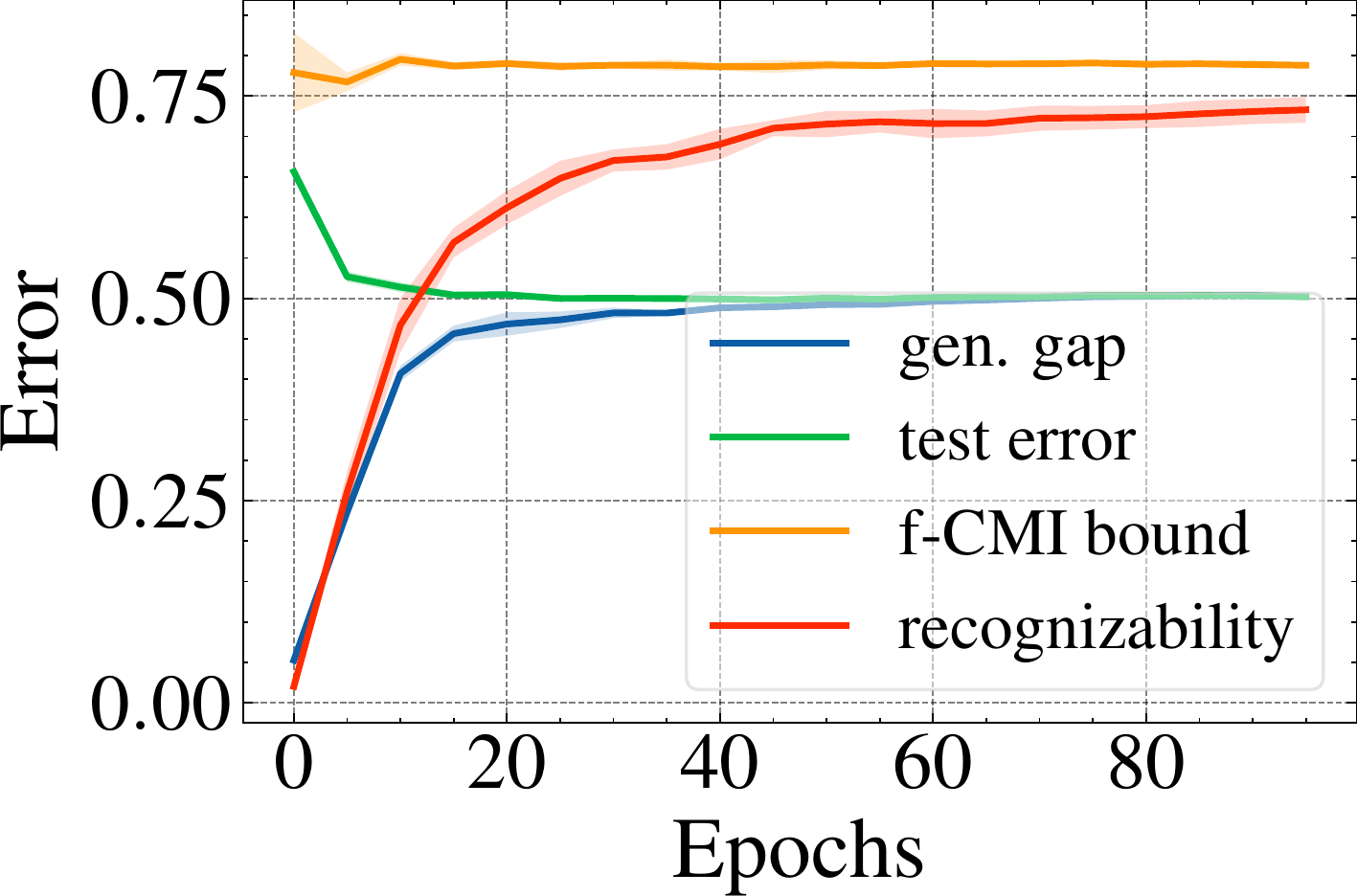}
        \caption{$n=5000$}
    \end{subfigure}
    \begin{subfigure}[b]{0.75\columnwidth}
        \centering
        \includegraphics[width=\textwidth]{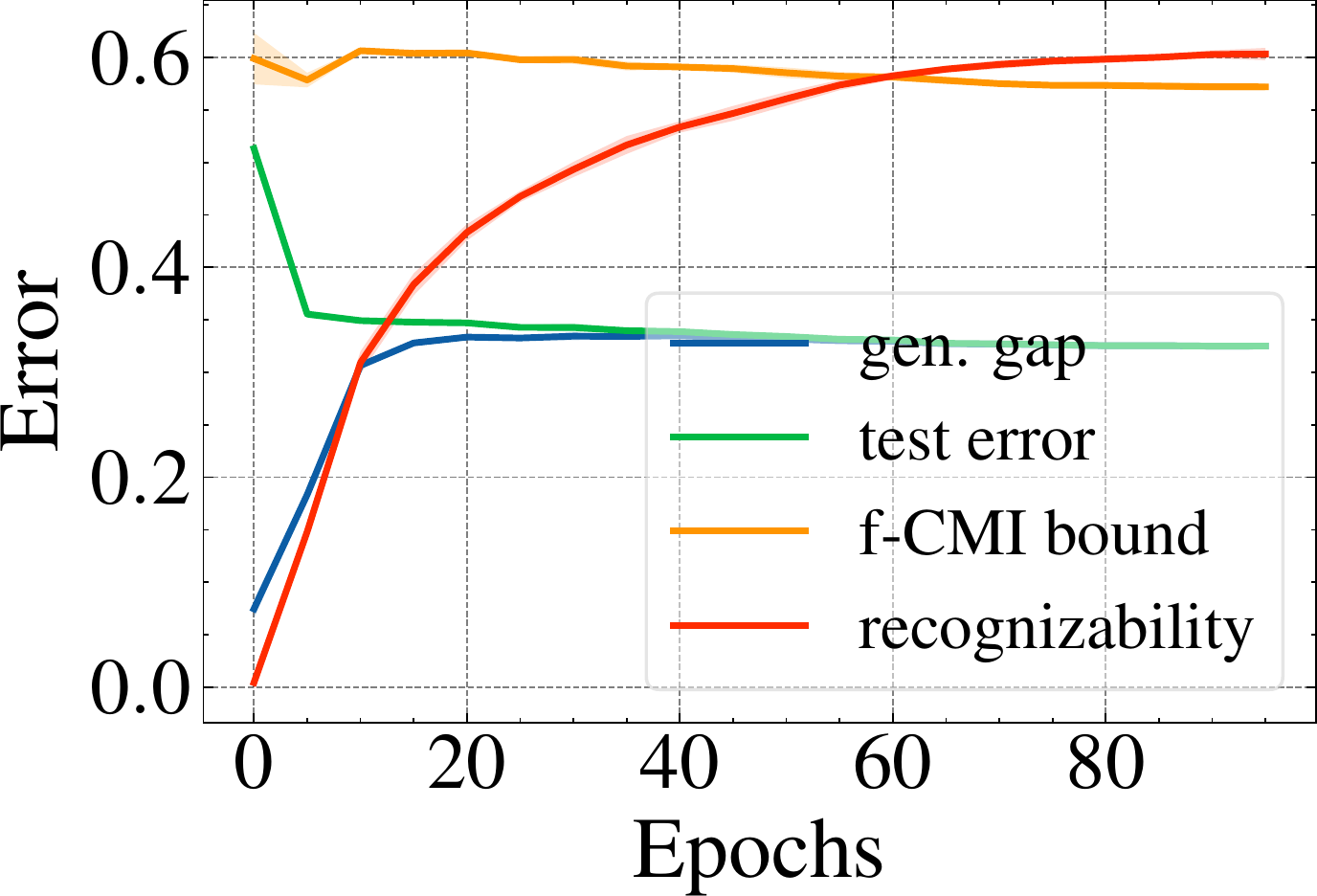}
        \caption{$n=20000$}
    \end{subfigure}
    \caption{Dynamics of recognizability on CIFAR10 with training set sizes of 5000 and 20000.}
    \label{fig:epoch}
\end{figure}

\section{Computational Complexity Analysis}
Table \ref{table:time} presents our computational complexity analysis, where $m$ is the number of data points, $d$ and $d^{\prime}$ are the dimension of the encoder network and the recognizability critic, respectively, and $c$ is the number of iterations used to estimate the Fisher information matrix. Note that the raw PIB objective requires a number of approximations to make the computational complexity affordable, due to the use of second-order information. Nevertheless, the computational complexity of the PIB is still higher than that of the RIB.
\begin{table}[ht]
	\centering
	\begin{adjustbox}{width=\columnwidth}
		\begin{tabular}{cccc}
			\toprule
			SGD               & PIB (w/o approx.)      & PIB                   & RIB                           \\
			\midrule
			$\mathcal{O}(md)$ & $\mathcal{O}(md+md^2)$ & $\mathcal{O}(md+cmd)$ & $\mathcal{O}(md+md^{\prime})$ \\
			\bottomrule
		\end{tabular}
	\end{adjustbox}
	\caption{Comparison of computational complexity per step.}
	\label{table:time}
\end{table}

\section*{Acknowledgments}
This work was partly supported by the National Natural Science Foundation of China under Grant 62176020; the National Key Research and Development Program (2020AAA0106800); the Beijing Natural Science Foundation under Grant L211016; the Fundamental Research Funds for the Central Universities (2019JBZ110); Chinese Academy of Sciences (OEIP-O-202004); Grant ITF MHP/038/20, Grant CRF 8730063, Grant RGC 14300219, 14302920.

\bibliographystyle{named}
\bibliography{ijcai23}

\end{document}